\documentclass{article}

\usepackage{amsmath,amsfonts,bm}

\def\eqref#1{equation~\ref{#1}}

\def\1{\bm{1}}

\DeclareMathAlphabet{\mathsfit}{\encodingdefault}{\sfdefault}{m}{sl}
\SetMathAlphabet{\mathsfit}{bold}{\encodingdefault}{\sfdefault}{bx}{n}

\usepackage{subcaption}

\usepackage{url}            
\usepackage{booktabs}       
\usepackage{amsfonts}       
\usepackage{nicefrac}       
\usepackage{microtype}      
\usepackage[dvipsnames,table]{xcolor}         
\usepackage{multirow}
\usepackage{amsmath}
\usepackage{amssymb}
\usepackage{mathtools}
\usepackage{amsthm}
\usepackage{graphicx}
\usepackage{algorithm}
\usepackage{algpseudocode}
\usepackage{makecell}

\usepackage{hyperref}

\usepackage[preprint]{icml2026}

\usepackage[capitalize,noabbrev]{cleveref}

\theoremstyle{plain}
\newtheorem{theorem}{Theorem}[section]

\newtheorem{lemma}[theorem]{Lemma}

\theoremstyle{definition}

\theoremstyle{remark}

\usepackage[textsize=tiny]{todonotes}

\icmltitlerunning{Submission and Formatting Instructions for ICML 2026}

\begin{document}

\twocolumn[
  \icmltitle{FlattenGPT: Depth Compression for Transformer with Layer Flattening}



  \icmlsetsymbol{equal}{*}

  \begin{icmlauthorlist}
    \icmlauthor{Ruihan Xu}{pku}
    \icmlauthor{Qingpei Guo}{ant}
    \icmlauthor{Yao Zhu}{thu}
    \icmlauthor{Xiangyang Ji}{thu}
    \icmlauthor{Ming Yang}{ant}
    \icmlauthor{Shiliang Zhang}{pku}
  \end{icmlauthorlist}

  \icmlaffiliation{pku}{State Key Laboratory of Multimedia Information Processing, School of Computer Science, Peking University}
  \icmlaffiliation{thu}{Tsinghua University}
  \icmlaffiliation{ant}{AntGroup}

  \icmlcorrespondingauthor{Shiliang Zhang}{slzhang.jdl@@pku.edu.cn}

  \icmlkeywords{Model Pruning, Structrued Pruning, Layer Pruning, Channel Pruning, LLM, ICML}

  \vskip 0.3in
]



\printAffiliationsAndNotice{}  

\begin{abstract}
Recent works have indicated redundancy across transformer blocks, prompting the research of depth compression to prune less crucial blocks.
However, current ways of entire-block pruning suffer from risks of discarding meaningful cues learned in those blocks, leading to substantial performance degradation.
As another line of model compression, channel pruning can better preserve performance, while it cannot reduce model depth and is challenged by inconsistent pruning ratios for individual layers.
To pursue better model compression and acceleration, this paper proposes \textbf{FlattenGPT}, a novel way to detect and reduce depth-wise redundancies.
By flatting two adjacent blocks into one, it compresses the network depth, meanwhile enables more effective parameter redundancy detection and removal.
FlattenGPT allows to preserve the knowledge learned in all blocks, and remains consistent with the original transformer architecture. Extensive experiments demonstrate that FlattenGPT enhances model efficiency with a decent trade-off to performance.
It outperforms existing pruning methods in both zero-shot accuracies and WikiText-2 perplexity across various model types and parameter sizes.
On LLaMA-2/3 and Qwen-1.5 models, FlattenGPT retains 90-96\% of zero-shot performance with a compression ratio of 20\%.
It also outperforms other pruning methods in accelerating LLM inference, making it promising for enhancing the efficiency of transformers.
\end{abstract}

\begin{figure*}[ht]
\centering
\includegraphics[width=0.98\linewidth]{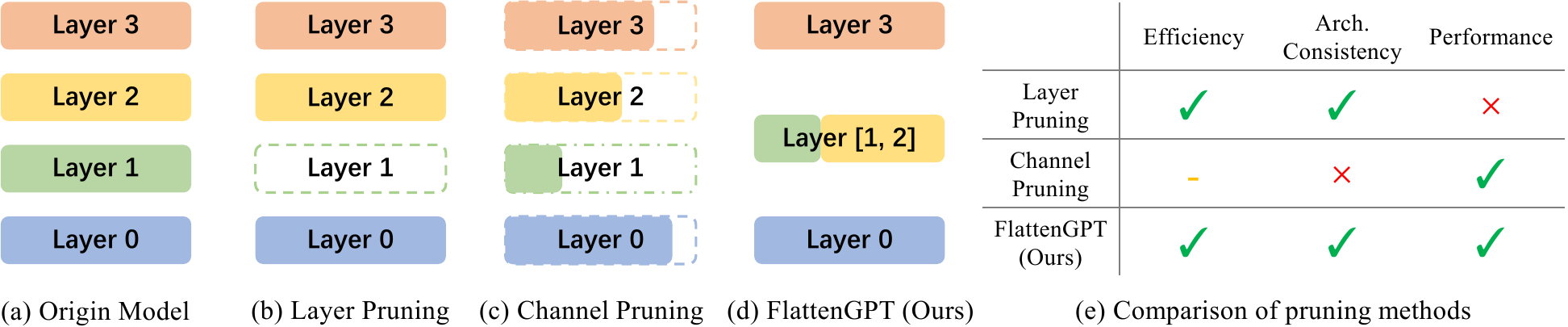}
\caption{Comparison of pruning methods. (a) The original architecture. (b) Layer pruning discards all knowledge in removed blocks. (c) Channel pruning cannot compress model depth and leads to inconsistent architecture across layers. (d) Our method bridges the gap, producing a compact model with marginal performance degradation. (e) compares efficiency, architecture consistency, and performance.}
\label{fig:flattengpt_intro}
\end{figure*}

\section{Introduction}
\label{sec:introduction}

Recent advancements in Large Language Models (LLMs)~\citep{gpt3,opt,palm,llama_v1,llama_v2,llama_v3} have led to breakthroughs in understanding and generation of natural language~\citep{hadi2023survey,zhao2023survey,minaee2024survey}.
However, the cost of heavy computation and extremely large memory consumption make it challenging to deploy on resource-limited devices. 
To mitigate these issues, model compression, as shown in Figure~\ref{fig:flattengpt_intro}, has emerged as a popular post-training solution, reducing model size and complexity by removing model redundancy~\citep{gupta2022compression,zhu2023survey}.

Depth compression~\citep{sleb,shortgpt} aims at reducing the redundancy across transformer blocks.
This redundancy manifests itself in the cross-layer similarity~\citep{ineffective_deep_layers,transformer_layers_as_painters,curse_of_depth}.
Figure~\ref{fig:layer_redundancy}(a) illustrates that the inputs of adjacent blocks have high similarity in LLMs, which is caused by the residual path spanning the entire LLM.
This similarity is particularly evident in LLMs, indicating that there is a certain amount of redundancy within them.
Depth compression methods aim to reduce this cross-layer redundancy to achieve a compact network architecture.
Compared to other pruning methods such as channel pruning~\citep{llm_pruner,slicegpt} or 2:4 pruning~\citep{sparsegpt,wanda}, depth compression methods have an evident advantage in inference speed with the same number of parameters~\citep{sleb}. 
However, previous depth compression methods usually adopt layer or block pruning, which removes the entire block selected by measuring how crucial this block is~\citep{shortgpt,rm,shortenedllama,sleb,blockpruner,finercut}.
It may simultaneously remove some useful knowledge learned in the pruned blocks, leading to substantial performance degradation.

As another line of model compression, channel pruning~\citep{llm_pruner,slicegpt,llm_surgeon,modegpt} conducts a fine-grained parameter preservation, thus leading to better performance.
However, these methods usually assign different pruning ratios for individual layers.
This inconsistency in the module architecture causes inconvenience in hyperparameter tuning or model deployment, such as LoRA hyperparameters~\citep{lora}.
Moreover, channel pruning cannot tackle the redundancy across layers, resulting in a deeper architecture and higher latency in practice.

This paper aims to conquer those issues, and proposes a fine-grained depth compression method called FlattenGPT.
FlattenGPT is composed of two stages to reduce the model depth, meanwhile preserve crucial knowledge.
In the first stage, we propose a new operation named flattening, to merge adjacent transformer blocks by concatenating their parameters and hidden states.
This operation changes the sequential execution of adjacent transformer blocks to parallel execution, where only the input of the blocks is altered.
Since the input features of some adjacent layers in LLMs are inherently of high similarity, flattening operation has a marginal impact on the model performance.
The subsequent stage employs a channel pruning method to streamline the merged transformer blocks.
Channel pruning can identify the critical channels within the merged blocks, 
allowing for a more efficient removal of redundancy.

FlattenGPT has clear advantages over previous pruning methods.
As shown in Figure~\ref{fig:flattengpt_intro}, unlike layer pruning methods, flattening preserves the knowledge embedded in each layer, raising the performance ceiling of the depth compression.
Compared with channel pruning, FlattenGPT yields a consistent architecture with lower depth, resulting in high efficiency as well as easy tuning and deployment.
This method bridges the gap between depth compression and channel pruning, allowing for comprehensive model compression.
Extensive experiments demonstrate that FlattenGPT preserves up to 96\% of zero-shot performance with a compression rate of 20\% on LLaMA 2~\citep{llama_v2}, outperforming prior depth compression approaches.
To the best of our knowledge, FlattenGPT is an original effort on transformer compression through layer flattening. It shows potential to establish a novel comprehensive paradigm that enhances the depth compression of transformer architectures.

\begin{figure*}[t]
\centering
\includegraphics[width=0.98\linewidth]{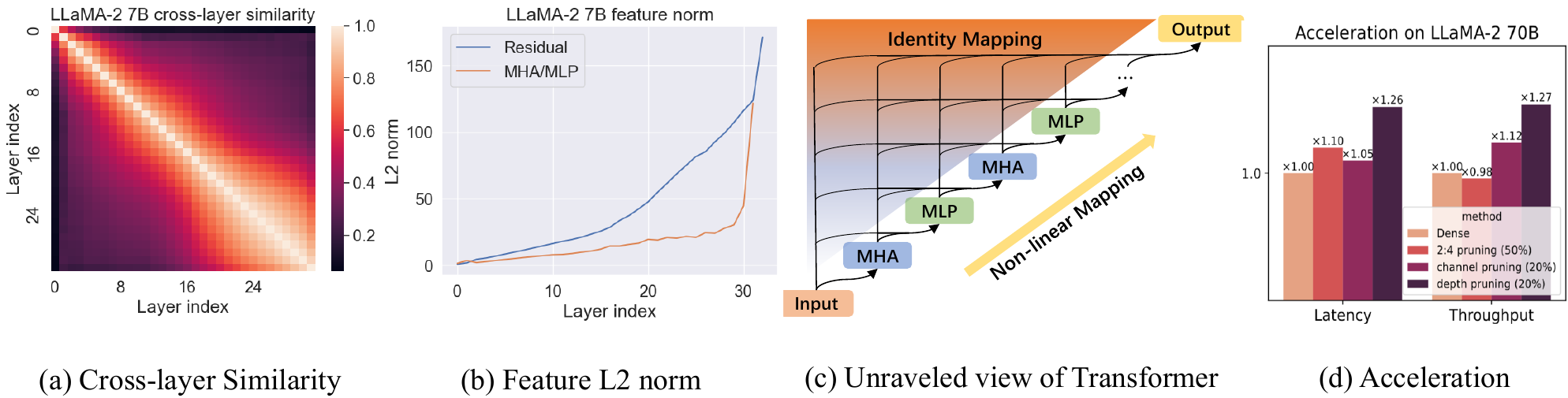}
\caption{Illustration of redundancy in transformer blocks. (a) LLaMA-2 7B exhibits high cross-layer similarity. (b) The scale of the residual path grows faster than the MHA/MLP blocks, which dominates the deep hidden states. (c) The unraveled view of transformer architecture, where the residual path traversing the entire network leads to a considerable cross-layer similarity. (d) The acceleration comparison between different pruning methods.}
\label{fig:layer_redundancy}
\end{figure*}

\section{Preliminary and Analysis}
\label{sec:preliminary}

\subsection{Preliminary of Transformer Architecture}
The Pre-LN transformer architecture in LLMs~\citep{llama_v1} consists of multiple decoder layers, each composed of two blocks, \emph{i.e.}, Multi-Head Attention (MHA) and Multi Layer Perceptron (MLP).
Let \(\ell\in\left\{0, 1, \cdots, L-1\right\}\) denote the layer index, \(T\), \(d\), \(d_{h}\), \(d_{int}\) and \(H\) denote the sequence length, hidden dimension, head dimension, intermediate dimension, and the number of attention heads, respectively.
The formulation of the \(\ell\)-th Transformer block \(B_{\ell}\) is denoted as
\begin{equation}
\begin{aligned}
    \tilde{\bm{H}}^{\ell-1}
    =&\bm{H}^{\ell-1} + \operatorname{MHA}^{\ell}\left(\operatorname{LN}_{a}^{\ell}\left(\bm{H}^{\ell-1}\right)\right),\\
    \bm{H}^{\ell}
    =&\tilde{\bm{H}}^{\ell-1} + \operatorname{MLP}^{\ell}\left(\operatorname{LN}_{p}^{\ell}\left(\tilde{\bm{H}}^{\ell-1}\right)\right),
\end{aligned}
\label{eq:transformer}
\end{equation}
where \(\bm{H}^{\ell}\in\mathbb{R}^{T\times d}\) denotes the output of the \(\ell\)-th layer, \(\operatorname{MHA}^{\ell}\), \(\operatorname{MLP}^{\ell}\), \(\operatorname{LN}_{a}^{\ell}\), and \(\operatorname{LN}_{p}^{\ell}\) denote the MHA block, MLP block, MHA normalization, and MLP normalization of the \(\ell\)-th Transformer layer, respectively.
The normalization layers are usually composed of a root mean square normalization and an element-wise affinement:
\begin{equation}
\begin{aligned}
    \operatorname{LN}_{a}^{\ell}\left(\bm{X}\right)
    =&\operatorname{RMSNorm}\left(\bm{X}\right)\operatorname{diag}\left(\bm{\alpha}_{a}^{\ell}\right),\\
    \operatorname{LN}_{p}^{\ell}\left(\bm{X}\right)
    =&\operatorname{RMSNorm}\left(\bm{X}\right)\operatorname{diag}\left(\bm{\alpha}_{p}^{\ell}\right),
\end{aligned}
\end{equation}
where \(\operatorname{RMSNorm}\left(\bm{X}\right)\) applies \(\bm{X}\leftarrow \bm{X}/\left\lVert \bm{X}\right\rVert\) to each row of \(\bm{X}\), \(\bm{\alpha}_{a}\in\mathbb{R}^{d}\) and \(\bm{\alpha}_{p}\in\mathbb{R}^{d}\) are the parameters of elementwise affinement.

\paragraph{\color{CornflowerBlue}{MHA block}}
Let \(\phi(\cdot)\) denote the softmax operation. The MHA block is defined as
\begin{equation}
\begin{aligned}
    &\operatorname{MHA}^{\ell}\left(\bm{X}\right)\\=&
    \sum_{i=1}^{H}
    \phi
    \left(
    \psi_{r}\left(\bm{X}\bm{W}_{Q,i}^{\ell}\right)
    \psi_{r}^{\top}\left(\bm{X}\bm{W}_{K,i}^{\ell}\right)
    \right)\
    \bm{X}\bm{W}_{V,i}^{\ell}\bm{W}_{O,i}^{\ell},
\end{aligned}
\end{equation}
where \(\bm{X}\in\mathbb{R}^{T\times d}\) is the input feature, \(\bm{W}_{Q,i}^{\ell}\), \(\bm{W}_{K,i}^{\ell}\), \(\bm{W}_{V,i}^{\ell}\in\mathbb{R}^{d\times{d_{h}}}\), and \(\bm{W}_{O,i}^{\ell}\in\mathbb{R}^{{d_{h}}\times d}\) denote the query, key, value, and output matrices of the \(i\)-th head in the \(\ell\)-th layer, respectively.
\(\psi_{r}(\cdot)\) denotes the positional embedding function.
For similicity, we denote \(\bm{W}_{Q} = \left[\bm{W}_{Q,1}\ \bm{W}_{Q,2}\ \cdots\ \bm{W}_{Q,H}\right]\) as the horizontal concatenation of query parameters from all heads, and similar to \(\bm{W}_{K}\) and \(\bm{W}_{V}\).
We denote \(\bm{W}_{O} = \left[\bm{W}_{O,1}^{\top}\ \bm{W}_{O,2}^{\top}\ \cdots\ \bm{W}_{O,H}^{\top}\right]^{\top}\) as the vertical concatenation of output parameters.

\paragraph{\color{Green}MLP block}
The MLP block is defined as
\begin{equation}
    \operatorname{MLP}^{\ell}\left(\bm{X}\right) =
    \psi_{s}\left(\bm{X}\bm{W}_{U}^{\ell}\right)\bm{W}_{D}^{\ell},
\end{equation}
where \(\bm{W}_{U}^{\ell}\in\mathbb{R}^{d \times d_{int}}\) and \(\bm{W}_{D}^{\ell}\in \mathbb{R}^{d_{int}\times d}\) denote the up and down matrix and \(\psi_{s}(\cdot)\) is the non-linear activation function.
\(\bm{X}\in\mathbb{R}^{T\times d}\) is the input matrix.
Prevailing LLMs~\citep{llama_v1,llama_v2,qwen} employ a gated MLP.
Its up matrix is composed of a up matrix and gate matrix \(\bm{W}_{U}^{\ell}=\left[\bm{W}_{u}\ \bm{W}_{g}\right]\), and the non-linear function is defined as \(\psi_{s}\left(\bm{X}\bm{W}_{U}^{\ell}\right)=\bm{X}\bm{W}_{u}^{\ell}\odot\psi_{g}\left(\bm{X}\bm{W}_{g}^{\ell}\right)\) where \(\psi_{g}\) is the gate activation.
For the following discussions, we take the gated MLP as the baseline architecture.

\subsection{Analysis on the Redundancy in Depth}
As illustrated in Figure~\ref{fig:layer_redundancy}(a), deep transformer architecture exhibits high cross-layer similarity.
This is caused by the curse of depth~\citep{curse_of_depth}, which implies that the deep layers are dominated by the residual path, \emph{i.e.}, identity mapping.
As shown in Figure~\ref{fig:layer_redundancy}(b), the L2 norm of the residual path is much larger than the MHA/MLP output in deep layers, dominating the forward propagation.
An intuitive interpretation is shown in the triangle-shaped unraveled view of transformer architecture in Figure~\ref{fig:layer_redundancy}(c).
The amount of residual features increases in deep layers and surpasses the non-linear blocks, leading to approximately identity mapping.
This analysis shows the cross-layer redundancy in the transformers.

To better illustrate the layer redundancy in deep transformers, the following part provides a theoretical analysis.
Assume that the input feature \(\bm{H}^{\ell}\), intermediate vectors \(\tilde{\bm{H}}^{\ell}\), and the model parameter matrix \(\bm{W}^{\ell}\) follow normal and independent distributions with mean \(0\) for all layers.
We first model the growth of the hidden states in a transformer architecture:
\begin{lemma}[The growth of the hidden state variance]
Let \(\sigma_{\bm{H}^{\ell}}^{2}\) denote the variance of \(\bm{H}^{\ell}\).
The variance \(\sigma_{\bm{H}^{\ell}}^{2}\) could increase quadratically with depth \(\ell\):
\begin{equation}
\sigma^2_{\bm{H}^{\ell}} = \sigma_{\bm{H}^{0}}^2 + \sum_{k=1}^{\ell} \Theta\left( \sigma_{\bm{H}^{k}}\right) \leq \Theta\left(\ell^2\right).
\end{equation}
\label{thm:variance_bound}
\end{lemma}
This upper bound is tight, so it is attainable. This conclusion indicates that deep transformers often come with imbalanced feature variance across layers, which is consistent with the empirical results in Figure~\ref{fig:layer_redundancy}(b) and Figure~\ref{fig:more_empirical_results} in the Appendix.
Then the following lemma further illustrates why deeper layers are redundant:
\begin{lemma}[The norm of gradient]
Let \(\frac{\partial y}{\partial \bm{H}^{\ell}}\) denote the partial derivative of the model output \(y\) to the \(\ell\)-th hidden states \(\bm{H}^{\ell}\).
The Euclidean norm of this partial derivative is bounded by
\begin{equation}
\left\| \frac{ \partial y}{\partial \bm{H}^{\ell}} \right\|_2 \leq \prod_{k=\ell}^{L} \left( 1 + \frac{1}{\sigma_{\bm{H}^{\ell}}} A + \frac{1}{\sigma_{\bm{H}^{\ell}}^2} B \right),
\end{equation}
where \(A\) and \(B\) are constants for the Transformer network. Specifically, when \(\ell=L-c\), where \(c\) is a constant number, the limitation of the right-hand side is 1.
\label{thm:gradient_bound}
\end{lemma}
This conclusion implies that for a large \(L\), the gradient of deeper layers \(\bm{H}^\ell\), \emph{i.e.}, \(\left\| \frac{\partial y}{\partial \bm{H}^{\ell}} \right\|_2\) is dominated by \textbf{identity mapping}, thereby limiting the model’s expressivity and hindering its capability to learn meaningful transformations.
This conclusion is also consistent with the empirical results as shown in Figure~\ref{fig:layer_redundancy}(a), where deeper layers exhibit high cross-layer similarity.
A comprehensive proof and more empirical results are given in Appendix~\ref{appendix:proof} and \ref{appendix:empirical}, respectively.

This redundancy in depth makes it reasonable that, previous layer pruning methods delete the entire redundant blocks, \emph{i.e.}, \(\operatorname{MHA}^{\ell}\) or \(\operatorname{MLP}^{\ell}\)~\citep{sleb,shortgpt}.
Although these methods achieve promising acceleration as shown in Figure~\ref{fig:layer_redundancy}(d), pruning at such high granularity will inevitably remove some useful knowledge within those blocks, resulting in a substantial performance degradation.

\begin{figure*}[t]
    \centering
    \includegraphics[width=0.98\linewidth]{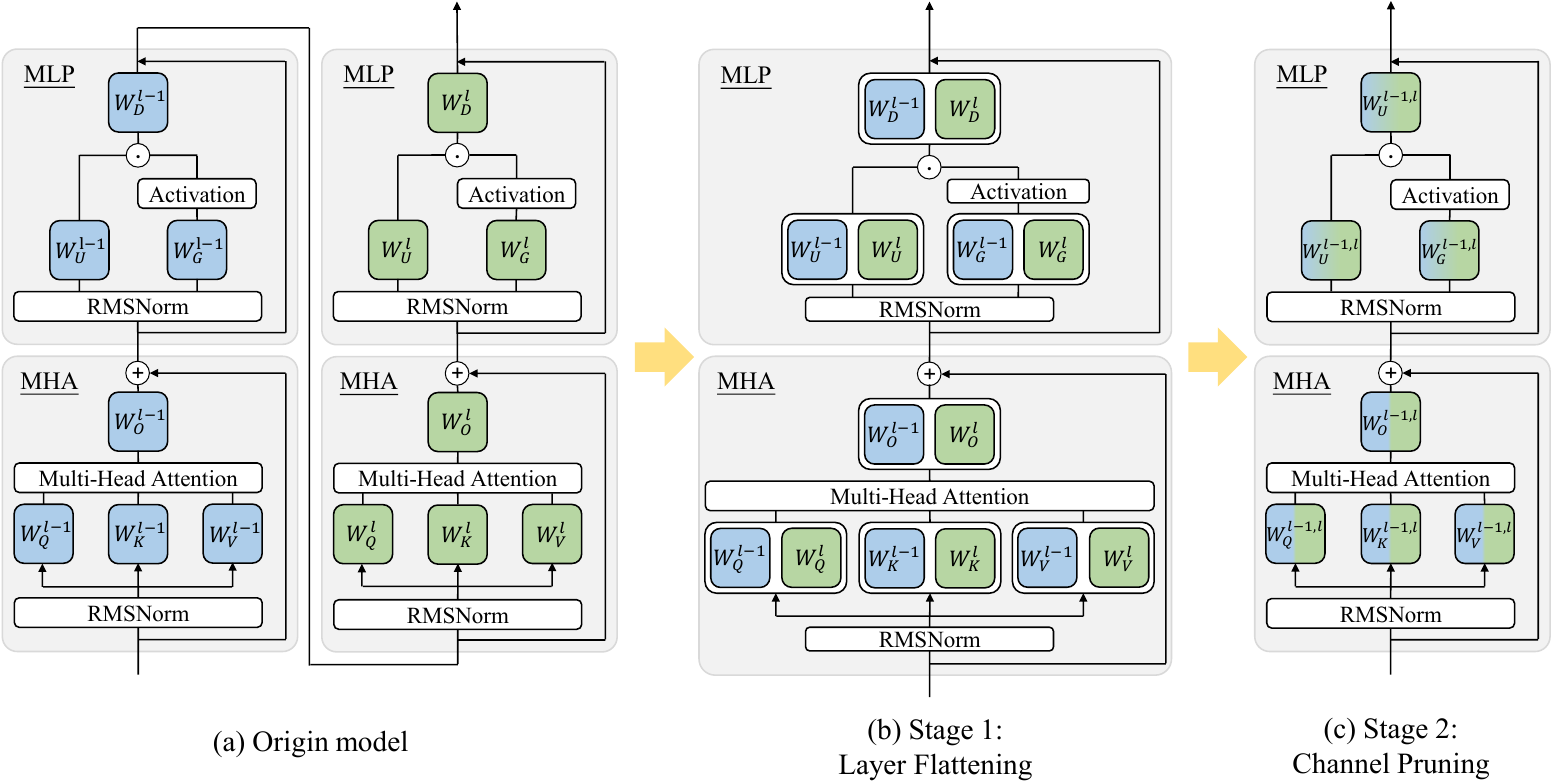}
    \caption{Framework of FlattenGPT, which consists of two stages. (a) Original stacks of transformer blocks with high similarity. (b) Layer flattening merges two adjacent blocks into one single block with a marginal performance degradation. (c) Flattening bridges the gap between depth compression and channel compression.}
    \label{fig:framework}
\end{figure*}

\section{FlattenGPT}
\label{sec:methods}

FlattenGPT strives for a fine-grained redundancy removal in depth compression.
As illustrated in Figure~\ref{fig:framework}, FlattenGPT employs a two-stage approach to compress the depth in a fine-grained manner.
In the first stage, FlattenGPT merges the selected adjacent layers into a single wide layer, flattening the arrangement of layers.
Due to the high similarity between adjacent layers, the flattening operation does not substantially alter the inner calculation, leading to a marginal performance degradation.
In the second stage, FlattenGPT adaptively prunes the redundant for the flattened layers, demonstrating less information loss compared with entire layer pruning methods.
FlattenGPT produces the same architecture as layer pruning, but preserves important cues from all layers.
It enjoys a faster acceleration in inference, meanwhile maintains high performance. Following parts present details of those two stages. 

\subsection{Iterative Layer Flattening}
Layer flattening aims to merge layers with high similarity.
Since the inputs of the two layers are highly similar, the inner calculation will not be significantly changed by the flattening, therefore preserving a constant performance.
We need to address two issues in this stage:
1) {how to select the layers to flatten}, and 2) {how to merge the selected layers}.

\noindent\textbf{Layer Selection} is conducted to spot highly similar adjacent layers. 
We collect cross-layer feature similarity on a small calibration dataset.
Then we design a greedy algorithm to find adjacent layers with the highest similarities iteratively.
The algorithm is shown in Algorithm~\ref{alg:iterative_layer_flattening}.
Let \(\bm{S}\in\mathbb{R}^{L\times L}\) denote the similarity matrix, where
\(\bm{S}_{i,j}\)
denote the cosine similarity between the input feature of layer \(i\) and the output feature of layer \(j\).
For simplicity, \(\bm{S}\) is an upper triangular matrix, where we only collect \(\bm{S}_{i,j}\) for \(i < j\).
We try to find the two adjacent layers \(\left\{B_{\ell-1}, B_{\ell}\right\}\) with the highest similarity \(\bm{S}_{\ell-1,\ell}\) for each iteration.
Then we need to modify the similarity matrix to \(\bm{S}\) for the next iteration.

Let layers \({B}_{\ell-1}\), \(\hat{B}_{\ell}\) are selected in the current iteration.
Suppose one of them has already been flattened in previous iterations, \emph{e.g.} \(\hat{B}_{\ell}\) is generated by merging \({B}_{m-1}\) and \({B}_{m}\). Merging \({B}_{\ell-1}\), \(\hat{B}_{\ell}\) is equivalent to flattening \({B}_{\ell-1}\), \({B}_{m-1}\), \({B}_{m}\), and thus requires special handling.
We need to consider the distance between the first and the last layer being flattened.
If the similarity is too low, flattening these layers will substantially alter information flow, potentially leading to performance degradation.
To address this issue, we modify the similarity matrix \(\bm{S}\) after each flattening iteration.
We delete the \(\ell - 1\)-th column and the \(\ell\)-th row in \(\bm{S}\) after merging layer \({B}_{\ell-1}\) and \(\hat{B}_{\ell}\).
Therefore, we can only access the similarities \(\bm{S}_{i,\ell}\) for \(i < \ell - 1\) and \(\bm{S}_{\ell-1,i}\) for \(i > \ell\).
It ensures that the similarity of \({B}_{\ell-1}\) and \({B}_{m}\) are counted for measuring the similarity between \({B}_{\ell-1}\) and \(\hat{B}_{\ell}\). Note that, this strategy is generalizable to merging either original or flattened layers.
The above steps are iterated until the target number of layers is flattened, which is decided by the predefined pruning ratio.

\begin{algorithm*}[t]
\footnotesize
\caption{{\color{RubineRed} \textbf{Iterative layer flattening}}}
\label{alg:iterative_layer_flattening}
\begin{algorithmic}[1]
\Require{Base model, number of layers to flatten \(N\), calibration set \(\mathcal{D}\)}
\State{Calculate the cosine similarity between each pair of layers \(\bm{S}\in\mathbb{R}^{L\times L}\)}, which is an upper triangle matrix
\While{\(N \ge 0\)}
\Comment{Iterative search}
\State{Identify the index \((\ell - 1, \ell)\) of the largest similarity in \(\bm{S}\)}
\Comment{Select adjacent layers}
\State{\(\tilde{\bm{W}}_{m}^{j} \leftarrow \operatorname{diag}\left(\bm{\alpha}_{a}^{j}\right)\bm{W}_{m}^{j}\), for \(m\in\{Q, K, V\}, j\in\{\ell-1, \ell\}\)}
\Comment{Fuse affinement parameters in MHA}
\State{\(\bm{W}_{m}^{\ell-1, \ell} \leftarrow \left[\tilde{\bm{W}}_{m}^{\ell-1}\ \tilde{\bm{W}}_{m}^{\ell}\right]\), for \(m\in\{Q, K, V\}\)}
\Comment{Flatten MHA}
\State{\(\bm{W}_{O}^{\ell-1, \ell} \leftarrow \left[\bm{W}_{O}^{\ell-1,\top}\ \bm{W}_{O}^{\ell,\top}\right]^{\top}\)}
\State{\(\tilde{\bm{W}}_{m}^{j} \leftarrow \operatorname{diag}\left(\bm{\alpha}_{p}^{j}\right)\bm{W}_{m}^{j}\), for \(m\in\{u, g\}, j\in\{\ell-1, \ell\}\)}
\Comment{Fuse affinement parameters in MLP}
\State{\(\bm{W}_{m}^{\ell-1, \ell} \leftarrow \left[\tilde{\bm{W}}_{m}^{\ell-1}\ \tilde{\bm{W}}_{m}^{\ell}\right]\), for \(m\in\{u, g\}\)}
\Comment{Flatten MLP}
\State{\(\bm{W}_{D}^{\ell-1, \ell} \leftarrow \left[\bm{W}_{D}^{\ell-1,\top}\ \bm{W}_{D}^{\ell,\top}\right]^{\top}\)}
\State{Delete the \(\ell\)-th row and the \(\ell-1\)-th column from the similarity matrix \(\bm{S}\)}
\Comment{Prepare for next iteration}
\State{\(N\gets N-1\)}
\EndWhile
\end{algorithmic}
\end{algorithm*}

\noindent\textbf{Layer Flattening:}
Flattening aims to merge the selected layers into a single wide layer.
To be compatible with the current implementation and AI infrastructure, we construct the flattened layers with the same architecture as the original transformers.
Let \(\ell-1\) and \(\ell\) to denote the indexes of layers being merged.
First, we fuse the parameters of affinement in the normalization layer \(\operatorname{LN}\) with the linear projections in the \(\operatorname{MHA}\) and \(\operatorname{MLP}\).
For the MHA layers, we fuse \(\bm{\alpha}_{a}^{\ell-1}\) with the query, key, and value matrix:
\begin{equation}
\begin{aligned}
    \tilde{\bm{W}}_{Q,i}^{\ell-1} = \operatorname{diag}\left(\bm{\alpha}_{a}^{\ell-1}\right)\bm{W}_{Q,i}^{\ell-1},\\
    \tilde{\bm{W}}_{K,i}^{\ell-1} = \operatorname{diag}\left(\bm{\alpha}_{a}^{\ell-1}\right)\bm{W}_{K,i}^{\ell-1},\\
    \tilde{\bm{W}}_{V,i}^{\ell-1} = \operatorname{diag}\left(\bm{\alpha}_{a}^{\ell-1}\right)\bm{W}_{V,i}^{\ell-1}.
\end{aligned}
\end{equation}
For the MLP layers, we fuse \(\bm{\alpha}_{p}^{\ell-1}\) with the up and gate matrix:
\begin{equation}
\begin{aligned}
    \tilde{\bm{W}}_{u}^{\ell-1} = \operatorname{diag}\left(\bm{\alpha}_{p}^{\ell-1}\right)\bm{W}_{u}^{\ell-1},\\
    \tilde{\bm{W}}_{g}^{\ell-1} = \operatorname{diag}\left(\bm{\alpha}_{p}^{\ell-1}\right)\bm{W}_{g}^{\ell-1}.
\end{aligned}
\end{equation}
Similar fuse operations are conducted on layer \(\ell\).
After these fusions, the affinement parameters are set to \(\bm{1}_{d}\).
This step does not change the output of the network but facilitates the next flattening steps.

\paragraph{\color{CornflowerBlue}{MHA flattening}}
Considering the high similarity between the selected layers, we add the output of the two MHA blocks in MHA flattening.
It is equivalent to concatenating the attention heads from two layers to form an attention block of \(2H\) heads:
\begin{equation}
\begin{aligned}
    &\operatorname{MHA}^{\ell-1,\ell}\left(\bm{X}\right)\\
    =&
    \sum_{i=1}^{H}
    \phi
    \left(
    \psi_{r}\left(\bm{X}\tilde{\bm{W}}_{Q,i}^{\ell-1}\right)
    \psi_{r}^{\top}\left(\bm{X}\tilde{\bm{W}}_{K,i}^{\ell-1}\right)
    \right)
    \bm{X}\tilde{\bm{W}}_{V,i}^{\ell-1}\bm{W}_{O,i}^{\ell-1}\\
    &+ \sum_{i=1}^{H}
    \phi
    \left(
    \psi_{r}\left(\bm{X}\tilde{\bm{W}}_{Q,i}^{\ell}\right)
    \psi_{r}^{\top}\left(\bm{X}\tilde{\bm{W}}_{K,i}^{\ell}\right)
    \right)
    \bm{X}\tilde{\bm{W}}_{V,i}^{\ell}\bm{W}_{O,i}^{\ell}.
\end{aligned}
\end{equation}
The flattened MHA block can also be calculated by the original implementation, which utilizes a single query, key, value, and output projection matrix:
\begin{equation}
\begin{aligned}
    \bm{W}_{Q}^{\ell-1,\ell}=&
    \left(
    \tilde{\bm{W}}_{Q,1}^{\ell-1},\ 
    \tilde{\bm{W}}_{Q,2}^{\ell-1},\ 
    \cdots,
    \tilde{\bm{W}}_{Q,H}^{\ell-1},\right.\\
    &\left.\tilde{\bm{W}}_{Q,1}^{\ell},\ 
    \tilde{\bm{W}}_{Q,2}^{\ell},\ 
    \cdots,
    \tilde{\bm{W}}_{Q,H}^{\ell}
    \right),
\end{aligned}
\end{equation}
where similar calculation applies to matrix \(\tilde{\bm{W}}_{K}^{\ell-1,\ell}\), \(\tilde{\bm{W}}_{V}^{\ell-1,\ell}\) and \(\tilde{\bm{W}}_{O}^{\ell-1,\ell}\).

\paragraph{\color{Green}MLP flattening}
Similarly, we add the output of two MLP blocks due to the high similarity across layers.
It is equivalent to concatenating the hidden states from two layers to form an MLP layer of \(2d_{int}\) hidden channels.
\begin{equation}
\begin{aligned}
    \operatorname{MLP}^{\ell-1,\ell}\left(\bm{X}\right) =&
    \bm{X}\tilde{\bm{W}}_{u}^{\ell-1}\psi_{g}\left(\bm{X}\tilde{\bm{W}}_{g}^{\ell-1}\right)\bm{W}_{D}^{\ell-1}\\
    + &\bm{X}\tilde{\bm{W}}_{u}^{\ell}\psi_{g}\left(\bm{X}\tilde{\bm{W}}_{g}^{\ell}\right)\bm{W}_{D}^{\ell}.
\end{aligned}
\end{equation}
The full procedure of layer flattening is shown in Algorithm~\ref{alg:iterative_layer_flattening}.
These steps flatten some adjacent layers, but do not change the number of parameters and calculations. Next, we introduce a pruning method to compress those parameters.

\begin{algorithm*}[t]
\footnotesize
\caption{{\color{CornflowerBlue} \textbf{MHA pruning}} by removing heads}
\label{alg:mha_pruning}
\begin{algorithmic}[1]
\Require{Query, key, value matrix \(\bm{W}_{Q}, \bm{W}_{K}, \bm{W}_{V}\in\mathbb{R}^{d\times d_{h}}\), output matrix \(\bm{W}_{O}\in\mathbb{R}^{d_{h}\times d}\), rank \(k\), calibration dataset \(\mathcal{D}\)}
\State{ \(f_i = \sum\limits_{i=1}^{N}\operatorname{Softmax}
    \left(
    \psi_{r}\left(\bm{X}\bm{W}_{Q,i}\right)
    \psi_{r}^{\top}\left(\bm{X}\bm{W}_{K,i}\right)
    \right)
    \bm{X}\bm{W}_{V,i}\) \(
    \operatorname{diag}\left(\bm{W}_{O,i}\bm{W}_{O,i}^{\top}\right)^{1/2}\), for \(i\in\left\{1, 2, \cdots, H\right\}\)}
\State{Let \(\bm{S}_{k}\in \mathbb{R}^{d\times k}\) be the matrix that selects the top \(k\) heads based on \(\left\{f_i\right\}\) scores}
\State{\Return \(\left(\bm{W}_{Q}, \bm{W}_{K}, \bm{W}_{V}, \bm{W}_{O}\right)\leftarrow\left(\bm{W}_{Q}\bm{S}_{k}, \bm{W}_{K}\bm{S}_{k}, \bm{W}_{V}\bm{S}_{k}, \bm{S}_{k}^{\top}\bm{W}_{O}\right)\)}
\end{algorithmic}
\end{algorithm*}

\begin{algorithm*}[t]
\footnotesize
\caption{{\color{Green} \textbf{MLP pruning}} by Nystr\"om approximation}
\label{alg:mlp_pruning}
\begin{algorithmic}[1]
\Require{Up and gate matrix \(\bm{W}_{u}, \bm{W}_{g}\in\mathbb{R}^{d\times d_{int}}\), down matrix \(\bm{W}_{D}\in\mathbb{R}^{d_{int}\times d}\), rank \(k\), calibration dataset \(\mathcal{D}\), ridge intensity \(\lambda\)}
\State{Calculate activation correlation \(\bm{C}_{\psi}=\sum_{i=1}^{N}\left(\bm{X}_{i}\bm{W}_{u}\psi_{g}\left(\bm{X}_{i}\bm{W}_{g}\right)\right)^{\top}\bm{X}_{i}\bm{W}_{u}\psi_{g}\left(\bm{X}_{i}\bm{W}_{g}\right)\)}
\State{\(s_{i}\leftarrow\left[\bm{C}_{\psi}\left(\bm{C}_{\psi} + \lambda\bm{I}\right)\right]^{-1}_{ii}\), for \(i\in\left\{1, 2, \cdots, d_{int}\right\}\)}
\Comment{Calculate the ridge leverage score}
\State{Let \(\bm{S}_{k}\in \mathbb{R}^{d_{int}\times k}\) be the matrix that selects the top \(k\) columns based on \(s_i\) scores}
\State{\Return \(\left(\bm{W}_{u}, \bm{W}_{g}, \bm{W}_{D}\right)\leftarrow\left(\bm{W}_{u}\bm{S}_{k}, \bm{W}_{g}\bm{S}_{k}, \bm{W}_{D} + \left(\bm{S}_{k}^{\top}\bm{C}_{\psi}\bm{S}_{k} + \lambda\bm{I}\right)^{-1}\bm{S}_{k}^{\top}\bm{C}_{\psi}\left(\bm{I} - \bm{S}_{k}\bm{S}_{k}^{\top}\right)\bm{W}_{D}\right)\)}
\end{algorithmic}
\end{algorithm*}

\subsection{Channel Pruning} 
Channel pruning is conducted based on layer flattening to remove redundant parameters.
Previous channel pruning methods are compatible with layer flattening, but lead to inconsistent architectures across layers. 
Different, this paper aims to keep the pruned architecture consistent with the original one.
This consistency will simplify the implementation, facilitating the reuse of tuning hyperparameters and deployment.

We employ two pruning methods for MHA and MLP blocks, respectively.
For the MHA blocks, we prune redundant heads to keep the number of heads and head size unchanged.
For the MLP blocks, we prune individual channels with nystr\"om approximation~\citep{nystrom,nystrom_sampling}.
We omit the layer index in the formulation to simplify the description.

\paragraph{\color{CornflowerBlue}{MHA pruning}}
Since the flattened layer contains more heads than the original layer, we aim to prune heads to keep the number of heads the same as the other MHA blocks.
We design a metric to compare the importance of head \(i\), i.e., the \(f_i\) is computed as
\begin{equation}
\begin{aligned}
    f_i = &\mathbb{E}_{\mathcal{D}}\left[\operatorname{Softmax}
    \left(
    \psi_{r}\left(\bm{X}\bm{W}_{Q,i}\right)
    \psi_{r}^{\top}\left(\bm{X}\bm{W}_{K,i}\right)
    \right)\right.\\
    &\left.\bm{X}\bm{W}_{V,i} \
    \operatorname{diag}\left(\bm{W}_{O,i}\bm{W}_{O,i}^{\top}\right)^{1/2}\right].
\end{aligned}
\end{equation}
This metric estimates the expectation norm of the attention activation value by multiplying the L2 norm of each line of the output matrix.
It measures the impact of the head \(i\) on the output.
By comparing the impact of each head, we can remove the unimportant heads to prune the MHA block.
The compression procedure is summarized in Algorithm~\ref{alg:mha_pruning}.

\paragraph{\color{Green}MLP pruning}
The MLP blocks conduct important non-linear calculations in the transformer.
A simple channel selection is insufficient to maintain the performance of the original model.
We employ Nystr\"om approximation~\citep{nystrom,nystrom_sampling} to prune this block.
The compression method is summarized in Algorithm~\ref{alg:mlp_pruning}.
First, we calculate the ridge leverage score~\citep{nystrom_sampling} as the channel importance measurement.
Then we select the important channels and adjust the down matrix \(\bm{W}_{D}\) to compensate for the information loss with Nystr\"om approximation~\citep{nystrom}.
The following lemma illustrates that Nystr\"om approximation is an optimal estimation under least squares with L2 regularization. The proof is shown in Appendix~\ref{appendix:proof3}.
\begin{lemma}
Let \(\bm{S}_{k}\) denote a \(k\)-column selection matrix. Let \(\bm{C}_{\psi}\) denote the covariance \(\sum_{i=1}^{N}\left(\psi_{s}\left(\bm{X}_{i}\bm{W}_{U}\right)\right)^{\top}\psi_{s}\left(\bm{X}_{i}\bm{W}_{U}\right)\). The optimal estimation of \(\hat{\bm{W}}_{D}\) is defined by:
\begin{equation}
\begin{aligned}
\Delta\hat{\bm{W}}_{D}
&= \operatorname*{\arg\min}_{\Delta\bm{W}_{D}} \left\lVert \psi_{s}\left(\bm{X}_{i}\bm{W}_{U}\right)\bm{S}_{k}\left(\bm{S}_{k}^{\top}\bm{W}_{D} + \Delta\bm{W}_{D}\right)\right.\\
&\left.- \psi_{s}\left(\bm{X}_{i}\bm{W}_{U}\right)\bm{W}_{D}
\right\rVert_{2}+\lambda\left\lVert\Delta\bm{W}_{D}\right\rVert_{2}
,\\
\hat{\bm{W}}_{D} &= \bm{W}_{D} + \bm{S}_{k}\Delta\hat{\bm{W}}_{D},
\end{aligned}
\label{eq:nystrom_approximation}
\end{equation}
where \(\lambda\) is the coefficient for L2 regularization.
\(\Delta\hat{\bm{W}}_{D}\) has closed form solution:
\begin{equation}
\Delta\hat{\bm{W}}_{D}
=\left(\bm{S}_{k}^{\top}\bm{C}_{\psi}\bm{S}_{k} + \lambda\bm{I}\right)^{-1}\bm{S}_{k}^{\top}\bm{C}_{\psi}\left(\bm{I} - \bm{S}_{k}\bm{S}_{k}^{\top}\right)\bm{W}_{D}.
\end{equation}
\label{thm:nystrom_approximation}
\end{lemma}

\subsection{Pruning Hyperparameters}
All architectural hyperparameters, including width/head count, are predefined to preserve structure consistency with the original transformer block.
The target width of the pruned MLP is identical to the original MLP module, and the number of heads is the same as the original MHA module.
This setting ensures compatibility with the original AI infrastructure, including GPUs, CUDA kernels, multi-machine communication, inference engine, etc.
It is also a clear target for reproducibility.

\begin{table*}[t]
\scriptsize
\caption{Comparison of depth compression methods on WikiText-2 perplexity and zero-shot tasks.}
\label{tab:depth_compression_comparison}
\renewcommand{\arraystretch}{0.9}
\setlength{\tabcolsep}{8pt}
\centering
\begin{tabular}{l | l | c | c | c  c  c  c  c | c}
\toprule
\rowcolor{gray!20}
\bf Model & \bf Method & \bf Sparsity	& \bf PPL \(\downarrow\)	& \bf WinoG	& \bf HellaS	& \bf PIQA	& \bf ARC-e	& \bf ARC-c	& \bf Avg.  \\
\midrule
\midrule
\multirow{7.5}{*}{LLaMA-2 7B}	& Dense	& 0\%		& 5.47	& 69.06	& 75.99	& 79.11	& 74.58	& 46.25	& 69.00	\\
\cmidrule{2-10}
& SLEB~\citep{sleb}					& 21.02\%	& 9.14	& 58.96	& 62.47	& 73.07	& 56.48	& 33.02	& 56.80	\\
& LaCo~\citep{laco}					& 21.02\%	& 50.39	& 60.46	& 54.08	& 68.34	& 55.39	& 35.84	& 54.82	\\
& RM~\citep{rm}						& 21.02\%	& 676.8	& 49.25	& 29.22	& 54.46	& 34.43	& 22.53	& 37.98	\\
& ShortGPT~\citep{shortgpt}			& 21.02\%	& 18.45	& 65.90	& 62.63	& 70.24	& 56.06	& 36.09	& 58.18	\\
& BlockPruner~\citep{blockpruner}	& 21.99\%	& 11.51	& 62.43	& 65.87	& 74.21	& 61.07	& 37.29	& 60.17	\\
\rowcolor{cyan!10}
\cellcolor{white}	& FlattenGPT	& 21.02\%	& 8.68	& 66.54	& 68.45	& 72.74	& 63.43	& 41.30	& \bf 62.49	\\
\midrule
\midrule
\multirow{6.5}{*}{LLaMA-2 13B}& Dense	& 0\%		& 4.88	& 72.22	& 79.39	& 80.47	& 77.48	& 49.23	& 71.76	\\
\cmidrule{2-10}
& LaCo~\citep{laco}	& 24.37\%		& 13.97	& 59.27	& 60.44	& 72.42	& 54.34	& 34.56	& 56.21	\\
& RM~\citep{rm}						& 24.37\%	& 10.08	& 66.61	& 66.80	& 73.72	& 66.12	& 41.98	& 63.05	\\
& ShortGPT~\citep{shortgpt}			& 24.37\%	& 20.06	& 70.80	& 67.80	& 72.74	& 60.35	& 41.30	& 62.60	\\
& BlockPruner~\citep{blockpruner}	& 25.12\%	& 8.16	& 66.30	& 72.20	& 76.93	& 65.82	& 41.38	& 64.53	\\
\rowcolor{cyan!10}
\cellcolor{white}	& FlattenGPT	& 24.37\%	& 6.68	& 71.11	& 73.44	& 76.33	& 72.10	& 44.54	& \bf 67.50	\\
\midrule
\midrule
\multirow{6.5}{*}{Qwen-1.5 7B}& Dense	& 0\%		& 7.95	& 66.46	& 76.92	& 79.22	& 62.16	& 42.66	& 65.48	\\
\cmidrule{2-10}
& LaCo~\citep{laco}					& 20.97\%	& 39.23	& 58.64	& 56.35	& 70.40	& 46.89	& 32.85	& 53.03	\\
& RM~\citep{rm}						& 20.97\%	& 2026	& 49.88	& 42.00	& 67.36	& 54.17	& 28.58	& 48.40	\\
& ShortGPT~\citep{shortgpt}			& 20.97\%	& 49.88	& 62.12	& 58.87	& 69.53	& 43.60	& 32.17	& 53.26	\\
& BlockPruner~\citep{blockpruner}	& 21.83\%	& 20.58	& 55.56	& 59.31	& 71.71	& 53.70	& 33.28	& 54.71	\\
\rowcolor{cyan!10}
\cellcolor{white}	& FlattenGPT	& 20.97\%	& 16.05	& 59.27	& 62.89	& 68.39	& 56.99	& 37.46	& \bf 57.00	\\
\midrule
\midrule
\multirow{6}{*}{Baichuan-2 7B}	& Dense	& 0\%		& 6.04	& 68.27	& 72.18	& 77.48	& 72.98	& 42.75	& 66.73	\\
\cmidrule{2-10}
& LaCo~\citep{laco}						& 21.57\%	& 26.46	& 58.56	& 51.50	& 68.28	& 52.90	& 28.50	& 51.95	\\
& RM~\citep{rm}							& 21.57\%	& 189.8	& 52.33	& 30.87	& 59.96	& 38.17	& 23.63	& 40.99	\\
& ShortGPT~\citep{shortgpt}				& 21.57\%	& 31.05	& 62.67	& 50.01	& 63.71	& 47.31	& 30.72	& 50.88	\\
& BlockPruner~\citep{blockpruner}		& 22.45\%	& \bf 15.38	& 61.48	& 58.09	& 69.75	& 58.08	& 33.02	& 56.08	\\
\rowcolor{cyan!10}
\cellcolor{white}		& FlattenGPT	& 21.57\%	& 20.55	& 64.33	& 61.50	& 69.42	& 56.27	& 35.24	& \bf 57.35	\\
\bottomrule
\end{tabular}
\end{table*}

\section{Experiments}
\label{sec:experiments}

\subsection{Experimental Setup}
\noindent\textbf{Models:}
We evaluate FlattenGPT on models that employ a sequential transformer block structure: LLaMA-2~\citep{llama_v2}, LLaMA-3~\citep{llama_v3}, Qwen-1.5~\citep{qwen}, and Baichuan-2~\citep{baichuan2}, etc.
These models share similar architectures of MHA and MLP.

\noindent\textbf{Implementations and Environments:}
All hyperparameters, including width/head count, are predefined to preserve structure consistency with the original transformer block.
We implement our models using the HuggingFace Transformers library~\citep{huggingface-transformers}.

\noindent\textbf{Datasets and Evaluations:}
We follow the setup in previous works~\citep{slicegpt,sleb} for fairness.
The calibration dataset is composed of 128 samples with 2048 tokens, randomly selected from the training split of WikiText-2~\citep{wikitext2}.
The evaluation consists of perplexity and zero-shot task performance.
The perplexity is evaluated on the test split of WikiText-2~\citep{wikitext2}.
The zero-shot accuracies are evaluated with LM Evaluation Harness~\citep{eval-harness} on Winograd~\citep{winogrande}, HellaSwag~\citep{hellaswag}, Physical Interaction Question Answering (PIQA)~\citep{piqa}, and AI2 Reasoning Challenges (ARC-e, ARC-c)~\citep{arc}.
We also investigate the effectiveness of recovery finetuning, which employs 50K samples of refined Alpaca~\citep{alpaca} for instruction tuning with LoRA~\citep{lora}.
More details are presented in Appendix~\ref{appendix:implementation}.


\subsection{Comparison with Depth Compression Methods}
Table~\ref{tab:depth_compression_comparison} shows a comprehensive comparison between FlattenGPT and the other depth compression methods.
These methods~\citep{laco,sleb,shortgpt,rm,blockpruner} remove the entire transformer blocks, resulting in substantial performance loss.
Our method alleviates this problem by fine-grained parameter removal and shows superior performance on various model sizes.
It achieves the highest zero-shot accuracy among training-free layer pruning competitors, and outperforms ShortGPT by about 5\% on various models.

\begin{table*}[t]
\centering
\scriptsize
\begin{minipage}[t!]{0.55\linewidth}
\centering
\caption{Comparison of pruning methods on throughput, latency, and mean accuracies on zero-shot tasks. Throughput and latency are measured with LLaMA-2 70B on 2 NVIDIA A100 80GB.}
\label{tab:latency_throughtput}
\renewcommand{\arraystretch}{0.9}
\setlength{\tabcolsep}{3pt}
\begin{tabular}{c|c|c|c|c|ccc}
\toprule
\rowcolor{gray!20}
\multicolumn{2}{c|}{}	& ~	& \bf Throughput	& \bf Latency	& \multicolumn{3}{c}{\bf LLaMA-2}	\\
\rowcolor{gray!20}
\multicolumn{2}{c|}{\multirow{-2}{*}{\bf Method}}	& \multirow{-2}{*}{\bf Sparsity}	& \bf (Tokens/s)	& \bf (ms)	& \bf 7B	& \bf 13B	& \bf 70B	\\
\midrule
\midrule
\multicolumn{2}{c|}{Dense}			& 0\%	& 299\(_{1.00\times}\)	& 1718.4\(_{1.00\times}\)	& 69.00	& 71.76	& 76.57	\\
\midrule
\multirow{2}{*}{2:4}	& SparseGPT	& 50\%	& 293{\color{red}\(_{0.98\times}\)}	& 1555.5{\color{Green}\(_{1.10\times}\)}	& 58.23	& 63.06	& 71.87	\\
& Wanda								& 50\%	& 293{\color{red}\(_{0.98\times}\)}	& 1555.5{\color{Green}\(_{1.10\times}\)}	& 55.59	& 61.23	& 72.34	\\
\midrule
\multirow{4}{*}{Width} & LLM-Pruner	& 20\%	& 314{\color{Green}\(_{1.05\times}\)}	& 1534.3{\color{Green}\(_{1.12\times}\)}	& 62.15	& 67.72	& -		\\
& SliceGPT							& 20\%	& 314{\color{Green}\(_{1.05\times}\)}	& 1658.7{\color{Green}\(_{1.04\times}\)}	& 58.17	& 63.45	& 72.34	\\
& SliceGPT							& 25\%	& 331{\color{Green}\(_{1.11\times}\)}	& 1440.7{\color{Green}\(_{1.19\times}\)}	& 55.49	& 58.90	& 69.75	\\
& SliceGPT							& 30\%	& 343{\color{Green}\(_{1.15\times}\)}	& 1364.2{\color{Green}\(_{1.26\times}\)}	& 51.50	& 55.16	& 66.11	\\
\midrule
& SLEB		& 10\%	& 336{\color{Green}\(_{1.12\times}\)}	& 1529.1{\color{Green}\(_{1.12\times}\)}	& 62.24	& 66.77	& 73.14 \\
& SLEB		& 20\%	& 381{\color{Green}\(_{1.27\times}\)}	& 1364.1{\color{Green}\(_{1.26\times}\)}	& 56.80	& 62.96	& 70.81 \\
\rowcolor{cyan!10}
\multirow{-3}{*}{Depth}& FlattenGPT	& 20\%	& \bf 381{\color{Green}\(_{1.27\times}\)}	& \bf 1364.1{\color{Green}\(_{1.26\times}\)}	& \bf 62.49	& \bf 68.27	& \bf 73.94 \\
\bottomrule
\end{tabular}
\end{minipage}
\hspace{1mm}
\begin{minipage}[t!]{0.43\linewidth}
\centering
\caption{Comparison of mean zero-shot accuracies with recovery fine-tuning. The sparsity ratio is 20\% and $^\dagger$ indicates fine-tuned on Alpaca~\citep{alpaca} dataset.}
\label{tab:zeroshot_finetuned}
\renewcommand{\arraystretch}{0.94}
\setlength{\tabcolsep}{3pt}
\begin{tabular}{c|l|c|c|c}
\toprule
\rowcolor{gray!20}
\multicolumn{2}{c|}{}	& \bf LLaMA-2 	& \bf LLaMA-2  & \bf LLaMA-3   \\
\rowcolor{gray!20}
\multicolumn{2}{c|}{\multirow{-2}{*}{\bf Method}}	& \bf 7B  & \bf 13B  & \bf 8B  \\
\midrule
\multicolumn{2}{c|}{Dense}							& 69.00	& 71.76	& 73.08	\\
\midrule
\multirow{4}{*}{Width}	& Wanda-sp	& 64.53	& 67.37	& -		\\
& FLAP								& 59.51	& 64.70	& 36.03	\\
& LLM-Pruner						& 61.34	& 65.66	& 64.23	\\
& LLM-Pruner$^\dagger$				& 62.15	& 67.72	& 68.99	\\
\midrule
\multirow{5}{*}{Depth}	& SLEB		& 59.25	& 62.96	& -		\\
& Shortened LLaMA 					& 58.36	& 65.86	& 58.30	\\
& Shortened LLaMA$^\dagger$ 		& 61.91	& 68.81	& 66.72	\\
\rowcolor{cyan!10}
\cellcolor{white}	& FlattenGPT			& 63.83	& 68.27	& 66.21	\\
\rowcolor{cyan!10}
\cellcolor{white}	& FlattenGPT$^\dagger$	& \bf 66.24	& \bf 70.53	& \bf 70.43	\\
\bottomrule
\end{tabular}
\end{minipage}
\end{table*}

\subsection{Comparison with other Pruning Methods}
Table~\ref{tab:latency_throughtput} compares the latency, throughput, and mean accuracies on zero-shot tasks of the compressed LLaMA-2~\citep{llama_v2} models.
Even accompanied with hardware acceleration as reported in~\cite{sleb}, 2:4 pruning methods lead to minor speedup (\(1.10\times\)) and lower throughput (\(0.98\times\) with a sparsity ratio of 50\%.
Width pruning methods, such as SliceGPT~\citep{slicegpt}, are more hardware-friendly and speed up the pruned model, while still lagging behind depth pruning methods.
FlattenGPT inherits the advantages of acceleration in depth pruning and improves the performance.
Since the compressed model architecture of FlattenGPT is exactly the same as SLEB, the throughput and latency results are the same.
FlattenGPT outperforms all other methods in throughput (\(1.27\times\)), latency (\(1.26\times\)), and zero-shot tasks performance (about 5\% higher), yielding a better trade-off between speed and performance.

\subsection{Recovery Fine-Tuning}
The pruned model of FlattenGPT contains the useful information from all blocks, making it easier for the model to recover performance through Recovery Fine-Tuning (RFT).
Table~\ref{tab:zeroshot_finetuned} presents the mean accuracies of zero-shot tasks with and without RFT.
The results show that the model compressed by FlattenGPT maintains \(>96\%\) zero-shot performance of the dense model, substantially better than other depth compression methods and some width pruning methods.
Even without RFT, our method achieves comparable performance with RFT-based methods.
These results illustrate the effectiveness of FlattenGPT.

\noindent\textbf{More experiments}: We present more experiments and discussions in Appendix~\ref{appendix:experiments}, including additional experiments on various model types and sizes, more pruning methods~\citep{llm_surgeon,modegpt}, ablation studies, dependency on calibration dataset, and generalization beyond language modeling and transformer architectures.
Please kindly refer to this part.

\section{Related Works}
\label{sec:related_works}

\noindent\textbf{Model Pruning} is an approach to compress the number of parameters and calculations in a deep model.
Unstructured pruning~\citep{obd,obs,llm_surgeon,layer_wise_obs,obc,sparsegpt,wanda,DSnoT} removes independent weights without pre-determined patterns, leading to sparse weight matrices within the model.
This sparsity enables high pruning ratio but results in complex data access patterns, which are not conducive to hardware acceleration.
Structured pruning~\citep{slicegpt,llm_pruner,flap} removes elements to form dense matrices that are more efficiently processed by hardware.
These methods exhibit a remarkable acceleration but come with worse performance degradation.
This paper follows the structured pruning of LLMs, proposing a new depth compression method that balances performance and efficiency.

\noindent\textbf{Depth compression} aims to reduce the number of layers and speed up inference.
Layer pruning approaches use layer importance metrics to remove redundant layers from the model~\citep{shortgpt,rm,shortenedllama,sleb,blockpruner,finercut}.
These methods remove the weights and knowledge of the entire layer, limiting the performance of the pruned model.
Layer merging methods fuse the parameters of different layers by addition~\citep{laco,mka,slm}.
While this type of method uses information from different layers, simple addition can cause sharp performance degradation.
LLM-Streamline~\citep{llmstreamline}
This paper proposes FlattenGPT, a novel depth compression method that rearranges the layers, which reduces the model depth while retaining the information of each layer and maintains the performance well.

\noindent\textbf{Width Compression} reduces the number of parameters by reducing the width of the network.
LLM-Pruner~\citep{llm_pruner} uses gradient magnitudes to estimate the importance of neurons and efficiently fine-tune the performance of the recovery model with parameters.
SliceGPT~\citep{slicegpt} and ModeGPT~\citep{modegpt} employ matrix decomposition to compress the width of each block in the Transformer.
The attention head pruning and sharing methods~\citep{michel2019sixteen} were used to reduce the width of the attention module.
However, these methods cannot compress model depth, leading to higher inference latency.
Our method bridges the gap between channel pruning and layer pruning, which provides a fine-grained layer pruning method and improves the performance.

\section{Conclusion}
We propose a novel LLM depth compression method, FlattenGPT, to address the challenges of performance degradation under high-granularity layer pruning.
Upon the high similarity of cross-layer input features, we design a layer flattening operation to reduce the model depth with minimal performance loss.
Then we adopt channel pruning methods to reduce the number of parameters and calculations in the model.
Our proposed method performs well on LLM depth compression, showcasing the effectiveness of fine-grained depth compression.
We hope this work can inspire more future efforts in depth compression on neural architectures from the perspective of layer flattening.








\section*{Impact Statement}
This paper presents work whose goal is to advance the field of machine learning. There are many potential societal consequences of our work, none of which we feel must be specifically highlighted here.





\bibliography{example_paper}
\bibliographystyle{icml2026}

\newpage
\appendix
\onecolumn
\section{Proofs}
\label{appendix:proof}

\subsection{Proof of Lemma~\ref{thm:variance_bound}}
\begin{proof}
Let $\sigma_{\text{Attn}}$ denote the standard deviation of
$\operatorname{Attn}\left(\operatorname{LN}_{a}^{\ell}\left(\bm{H}^{\ell}\right)\right)$,
and $\sigma_{\text{MLP}}$ denote the standard deviation of
$\operatorname{MLP}\left(\operatorname{LN}_{p}^{\ell}\left(\bm{H}^{\ell}\right)\right)$.
Given \eqref{eq:transformer} we have:
\begin{equation}
\begin{aligned}
\operatorname{Var}\left(\tilde{\bm{H}}^{\ell}\right)
&= \operatorname{Var}\left(\bm{H}^{\ell}\right) + \operatorname{Var}\left(\operatorname{Attn}\left(\operatorname{LN}\left(\bm{H}^{\ell}\right)\right)\right) + 2\operatorname{Cov}\left(\operatorname{Attn}\left(\operatorname{LN}\left(\bm{H}^{\ell}\right)\right),\bm{H}^{\ell}\right) \\
&=\sigma_{\bm{H}^{\ell}}^2 + \sigma_{\text{Attn}}^{2} + 2\rho_1 \cdot \sigma_{\bm{H}^{\ell}} \cdot  \sigma_{\text{Attn}},
\end{aligned}
\end{equation}
\begin{equation}
\begin{aligned}
\operatorname{Var}\left(\bm{H}^{\ell+1}\right)
&= \operatorname{Var}\left(\tilde{\bm{H}}^{\ell}\right) + \operatorname{Var}\left(\operatorname{MLP}\left(\operatorname{LN}\left(\tilde{\bm{H}}^{\ell}\right)\right)\right) + 2\operatorname{Cov}\left(\operatorname{MLP}\left(\operatorname{LN}\left(\tilde{\bm{H}}^{\ell}\right)\right),\tilde{\bm{H}}^{\ell}\right) \\
&=\sigma_{\tilde{\bm{H}}^{\ell}}^2 + \sigma_{\operatorname{MLP}}^{2} +  2\rho_2 \cdot \sigma_{\tilde{\bm{H}}^{\ell}} \cdot \sigma_{\operatorname{MLP}},
\end{aligned}
\end{equation}

where $\rho_1$, $\rho_2$ is the correlation factor.
Thus, the evolve from $\operatorname{Var}\left(\bm{H}^{\ell+1}\right)$ to $\operatorname{Var}\left(\bm{H}^{\ell}\right)$ becomes
\begin{equation}
\sigma_{\bm{H}^{\ell+1}}^{2} =\sigma_{\bm{H}^{\ell}}^{2} + \sigma_{\text{Attn}}^2 + \sigma_{\operatorname{MLP}}^2 + 2\rho_1 \cdot \sigma_{\bm{H}^{\ell}} \cdot  \sigma_{\text{Attn}} + 2\rho_2 \cdot \sigma_{\tilde{\bm{H}}^{\ell}} \cdot  \sigma_{\operatorname{MLP}}.
\label{eq:layer_variance}
\end{equation}

First, we deal with the variance of attention.
Let $n$ denote the number of keys/values.
Following results in \citet{curse_of_depth}, based on the independent distribution assumption of weights, we have
\begin{equation}
\operatorname{Var}\left(\operatorname{Attn}\left(\bm{Q}, \bm{K}, \bm{V}\right)\right)
\sim \frac{1}{n}\sum_{i=1}^{n}\operatorname{Var}\left(\bm{V}_{i}\right)
=\frac{1}{n}\cdot n\cdot \sigma_{\bm{V}}^2
=\sigma_{\bm{V}}^{2}
=\sigma_{\bm{W}}^{2}.
\end{equation}
Then the variance of the attention module with residual connection is denoted as
\begin{equation}
\begin{aligned}
\sigma^2_{\tilde{\bm{H}}^{\ell}}
&=\sigma_{\bm{H}^{\ell}}^2 + \sigma_{\bm{W}}^{2} + 2\rho_{1} \cdot \sigma_{\bm{H}^{\ell}} \cdot \sigma_{\bm{W}} \\
&= \sigma_{\bm{H}^{\ell}}^2 \left( 1+ \frac{\sigma_{\bm{W}}^2}{\sigma_{\bm{H}^{\ell}}^2}+ 2\rho_{1} \cdot \frac{\sigma_{\bm{W}}}{\sigma_{\bm{H}^{\ell}}}\right)\\
& = \sigma_{\bm{H}^{\ell}}^2 \Theta\left(\left(1 + \frac{\sigma_{\bm{W}}}{\sigma_{\bm{H}^{\ell}}}\right)^2\right).
\label{prime_or_not}
\end{aligned}
\end{equation}
Second, we solve the variance of MLP. Using the conclusion obtained by \citet{deepnet}, we get
\begin{equation}
    \sigma_{\operatorname{MLP}}^{2}\sim \sigma_{\bm{W}_U}^2\cdot \sigma_{\bm{W}_D}^2 = \sigma_{\bm{W}}^4.
\end{equation}
Substitute $
\sigma_{\tilde{\bm{H}}_{\ell}} = \sigma_{\bm{H}^{\ell}} \Theta\left( 1 + \frac{\sigma_{\bm{W}}}{\sigma_{\bm{H}^{\ell}}}\right)
$, we can obtain the variance of \eqref{eq:layer_variance}:
\begin{equation}
\begin{aligned}
\sigma^2_{\bm{H}^{\ell+1}} &=\sigma^2_{\bm{H}^{\ell}} +  \sigma_{\bm{W}}^2 +  \sigma_{\bm{W}}^4 + 2\rho_1 \cdot \sigma_{\bm{H}^{\ell}} \cdot \sigma_{\bm{W}} + 2\rho_2 \cdot \sigma_{\tilde{\bm{H}}^{\ell}} \cdot  \sigma_{\bm{W}}^2\\
& = \sigma^2_{\bm{H}^{\ell}} + \sigma_{\bm{W}}^2 + \sigma_{\bm{W}}^4 + 2\rho_1 \cdot \sigma_{\bm{H}^{\ell}} \cdot \sigma_{\bm{W}} + 2\rho_2 \cdot \sigma_{\bm{W}}^2 \cdot \sigma_{\bm{H}^{\ell}} \Theta\left(1 + \frac{\sigma_{\bm{W}}}{\sigma_{\bm{H}^{\ell}}}\right) \\
& = \sigma_{\bm{H}^{\ell}}^2 + \Theta\left( \sigma_{\bm{H}^{\ell}}\right).
\end{aligned}
\label{eq:variance_iteration}
\end{equation}
We set the numerator part to 1 in Big-Theta notation for simplicity. The variance with regard to \(\sigma_{\bm{H}}^{2}\) can be obtained by iteratively applying \eqref{eq:variance_iteration}:
\begin{equation}
\sigma_{\bm{H}^{\ell}}^{2}
= \sigma_{\bm{H}^{0}}^2 + \sum_{k=1}^{\ell} \Theta\left( \sigma_{\bm{H}^{k}}\right).
\end{equation}
To obtain the upper bound of the variance, we will show that if there is a constant number \(c\) such that \(c > 0\) and \(\rho_1 > c\), \(\rho_2 > c\), the variance grows quadratically with the depth \(\ell\). We use mathematical induction to prove the bound of the variance. Let \(\sigma_{\bm{H}^{k}}^{2} = \Theta\left(k^2\right)\), then
\begin{equation}
\sigma_{\bm{H}^{k + 1}}^{2}
= \sigma_{\bm{H}^{k}}^2 + \Theta\left( \sigma_{\bm{H}^{k}}\right)
= \Theta\left(k^2\right) + \Theta\left(k\right)
= \Theta\left((k + 1) ^ 2\right).
\end{equation}
If \(\rho_1=\rho_2=1\), the variance grows the most quickly, and it also falls in \(\Theta\left((k + 1) ^ 2\right)\).

\end{proof}
This lemma indicates that the variance of the Pre-Norm Transformer could increase with multiple residual connections, leading to a quadratic increase with model depth.
Thus, the prevailing transformers are potentially non-stable and result in redundancy as the layer goes deeper.

\subsection{Proof of Lemma~\ref{thm:gradient_bound}}
\label{appendix:proof2}
\begin{proof}
For an $L$-layered Pre-LN Transformer, the partial gradient to the $\ell$-th hidden states is given by the chain rule:
\begin{equation}
\frac{\partial y}{\partial \bm{H}^{\ell}} = \prod_{k=\ell}^{L-1} \left( \frac{\partial \bm{H}^{k+1}}{\partial \tilde{\bm{H}}^{k}} \cdot \frac{\partial \tilde{\bm{H}}^{k}}{\partial \bm{H}^{k}} \right).
\end{equation}
From \cite{curse_of_depth}, we know that
\begin{equation}
\frac{\partial \bm{H}^{k+1}}{\partial \tilde{\bm{H}}^{k}}
\le
1 + \frac{\sigma_{\bm{W}_{U}^{\ell}} \sigma_{\bm{W}_{D}^{\ell}}}{\sigma_{ \tilde{\bm{H}}^{k}} (\sqrt{d} + \sqrt{d_\mathrm{MLP}})^2}
= 1 + \frac{\sigma^2_{\ell}}{\sigma_{\tilde{\bm{H}}^{k}} (\sqrt{d} + \sqrt{d_\mathrm{MLP}})^2}. \label{pre_LN_FFN_result}
\end{equation}
From \cite{papaspiliopoulos2020high}, we get
\begin{equation}
\frac{\partial \tilde{\bm{H}}^{k}}{\partial \bm{H}^{k}}
\le
\left( 1 + 2dh\left( \sqrt{s} + 2 + \frac{1}{\sqrt{s}} \right) \frac{\sigma^2}{\sigma_{\bm{H}^{k}} }  \left(\sigma^2 d \sqrt{ d_{\mathrm{head}}} + \left( 1+ \sqrt{d_{\mathrm{head}}/d} \right) \right) \right),
\end{equation}
where $h$ denotes the number of heads and $s$ denotes the sequence length, respectively.
Following the proof from \cite{curse_of_depth}, the target equation can be expressed as
\begin{equation}
\left\lVert \frac{ \partial y}{\partial \bm{H}^{\ell}} \right\rVert_2 \leq \prod_{k=\ell}^{L} \left( 1 + \frac{1}{\sigma_{\bm{H}^{k}}} A + \frac{1}{\sigma_{\bm{H}^{k}}^2} B \right),
\end{equation}
where
\begin{equation}
A = \frac{\sigma^2}{(\sqrt{d} + \sqrt{d_\mathrm{FFN}})^2} + 2dh \left( \sqrt{s} + 2 + \frac{1}{\sqrt{s}} \right) \sigma^2 \left( d \sqrt{d_{\mathrm{head}}} + 1 + \sqrt{d_{\mathrm{head}}/d} \right),
\label{A_form}
\end{equation}
\begin{equation}
B = 2dh \left( \sqrt{s} + 2 + \frac{1}{\sqrt{s}} \right) \sigma^4 d \sqrt{d_{\mathrm{head}}}.
\label{B_form}
\end{equation}
This conclusion indicates that for the deep layers in the model, the partial gradient $\frac{ \partial y}{\partial \bm{H}^{\ell}}$ will be bounded.
Considering the growth in Lemma~\ref{thm:variance_bound}, the partial gradient will be bounded by
\begin{equation}
\begin{aligned}
\left\lVert \frac{ \partial y}{\partial \bm{H}^{\ell}} \right\rVert_2
&\leq \prod_{k=\ell}^{L} \left( 1 + \frac{1}{\sigma_{\bm{H}^{k}}} A + \frac{1}{\sigma_{\bm{H}^{k}}^2} B \right)\\
&\le \prod_{k=\ell}^{L} \left( 1 + \frac{1}{\sigma_{\bm{H}^{\ell}}} A + \frac{1}{\sigma_{\bm{H}^{\ell}}^2} B \right)\\
&= \left( 1 + \frac{1}{\sigma_{\bm{H}^{\ell}}} A + \frac{1}{\sigma_{\bm{H}^{\ell}}^2} B \right)^{L-\ell}.
\end{aligned}
\end{equation}
Let $L-\ell$ be a constant number $c$, which implies the $c$-th layer from the last.
As $\ell$ grows, $\sigma_{\bm{H}^{\ell}}$ will grow to infinity.
Then we get
\begin{equation}
\lim_{\ell\rightarrow +\infty}\left\lVert \frac{ \partial y}{\partial \bm{H}^{\ell}} \right\rVert_2
\le \lim_{\ell\rightarrow +\infty}\left( 1 + \frac{1}{\sigma_{\bm{H}^{\ell}}} A + \frac{1}{\sigma_{\bm{H}^{\ell}}^2} B \right)^{c}
=1.
\end{equation}
\end{proof}
This lemma shows that as the feature norm increases, the Pre-Norm Transformer cannot effectively propagate gradients except for the identity mappings.
Therefore, the model will be degenerated into an identity mapping dominated by residual connections, which is highly inefficient in language modeling.

\subsection{Proofs of Lemma~\ref{thm:nystrom_approximation}}
\label{appendix:proof3}

\begin{proof}
The original solution for linear regression with L2 regularization is defined as follows:
\begin{lemma}
Let $\bm{X} \in \mathbb{R}^{n \times p}$ be the design matrix, $\bm{y} \in \mathbb{R}^n$ be the response vector, and $\lambda > 0$ be the regularization parameter. The L2 regularized linear regression (Ridge Regression) minimizes the following objective function:
\begin{equation}
\operatorname{\arg\min}\limits_{\bm{\theta}} \|\bm{X}\bm{\theta} - \bm{y}\|_2^2 + \lambda \|\bm{\theta}\|_2^2,
\label{eq:regression_l2}
\end{equation}
where $\bm{\theta} \in \mathbb{R}^p$ is the coefficient vector. The closed-form solution for the optimal coefficient vector $\hat{\bm{\theta}}$ is given by:
\begin{equation}
\hat{\bm{\theta}} = (\bm{X}^\top \bm{X} + \lambda \bm{I})^{-1} \bm{X}^\top \bm{y},
\end{equation}
provided that the matrix $(\bm{X}^\top \bm{X} + \lambda \bm{I})$ is invertible. Here, $\bm{I}$ denotes the $p \times p$ identity matrix. This solution always exists for $\lambda > 0$, even when $\bm{X}^\top \bm{X}$ is singular.
\end{lemma}
To minimize \(L(\bm{\theta})\), take the partial derivative with respect to \(\bm{\theta}\) and set it to \(\bm{0}\).
The differential equation is listed as follows:
\begin{equation}
\frac{\partial L(\bm{\theta})}{\partial \bm{\theta}} = -2\bm{X}^T \bm{y} + 2 \left( \bm{X}^T \bm{X} + \lambda \bm{I} \right) \bm{\theta} = \bm{0}.
\end{equation}
Simplify by dividing by 2 and rearranging terms:
\begin{equation}
\left( \bm{X}^T \bm{X} + \lambda \bm{I} \right) \bm{\theta} = \bm{X}^T \bm{y}.
\end{equation}
Since \(\bm{X}^T \bm{X} + \lambda \bm{I}\) is positive definite (hence invertible) for \(\lambda > 0\), left-multiply both sides by its inverse:
\begin{equation*}
\hat{\bm{\theta}} = \left( \bm{X}^T \bm{X} + \lambda \bm{I} \right)^{-1} \bm{X}^T \bm{y}.
\end{equation*}
\end{proof}
As for \eqref{eq:nystrom_approximation}, we aimed at minimizing the difference between the original parameters and the adjusted ones. Thus, we add an L2 regularization on the change of parameters \(\Delta \bm{W}_{D}\).
Given that $\theta = \Delta\bm{W}_{D}$, $\bm{X} = \psi_{s}\left(\bm{X}_{i}\bm{W}_{U}\right)\bm{S}_{k}$, $\bm{y}=\psi_{s}\left(\bm{X}_{i}\bm{W}_{U}\right)\left(\bm{I} - \bm{S}_{k}\bm{S}_{k}^{\top}\right)\bm{W}_{D}$, we will get the solutions for the MLP pruning:
\begin{equation}
\Delta\hat{\bm{W}}_{D}
=\left(\bm{S}_{k}^{\top}\bm{C}_{\psi}\bm{S}_{k} + \lambda\bm{I}\right)^{-1}\bm{S}_{k}^{\top}\bm{C}_{\psi}\left(\bm{I} - \bm{S}_{k}\bm{S}_{k}^{\top}\right)\bm{W}_{D},
\end{equation}
where \(\bm{C}_{\psi}=\sum_{i=1}^{N}\left(\psi_{s}\left(\bm{X}_{i}\bm{W}_{U}\right)\right)^{\top}\psi_{s}\left(\bm{X}_{i}\bm{W}_{U}\right)\) denotes the covariance.
This equation ensures the difference in the MLP output with regularization is minimized during the structured pruning.

\newpage
\section{Empirical Results}
\label{appendix:empirical}

We present more empirical results in various architectures, including LLaMA-2~\citep{llama_v2} at \{7B, 13B\}, Qwen-1.5~\citep{qwen} at \{7B, 14B\}, and Baichuan-2 at \{7B, 13B\}~\citep{baichuan2}.
As shown in Figure~\ref{fig:more_empirical_results}, the high cross-layer similarity and large feature norm is consistent across various model types and parameter sizes.
According to the theoretical analysis above, this phenomenon is deeply related to the architecture of transformers.
\begin{figure}[h]
    \centering
    \includegraphics[width=0.95\linewidth]{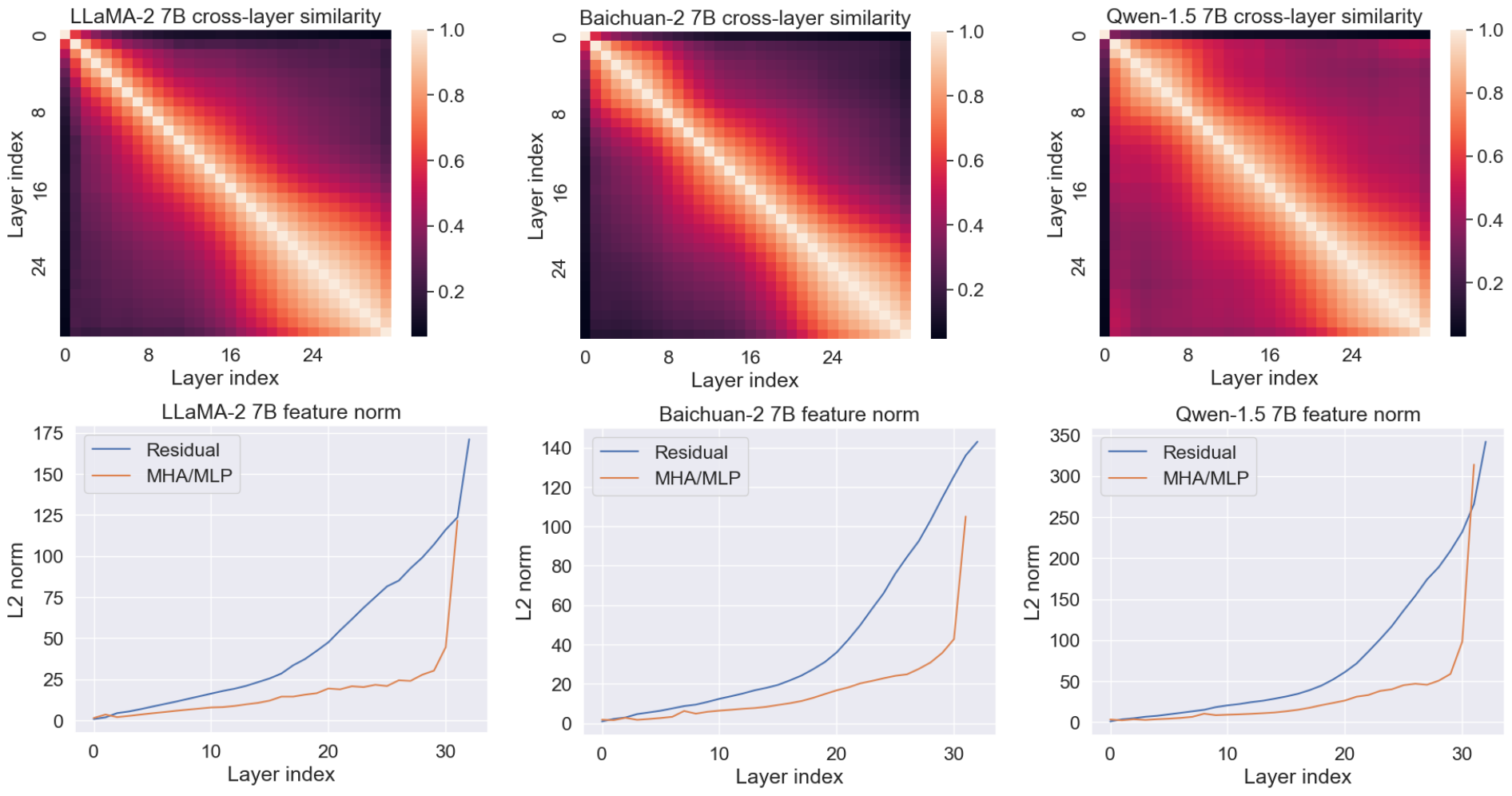}
    \includegraphics[width=0.95\linewidth]{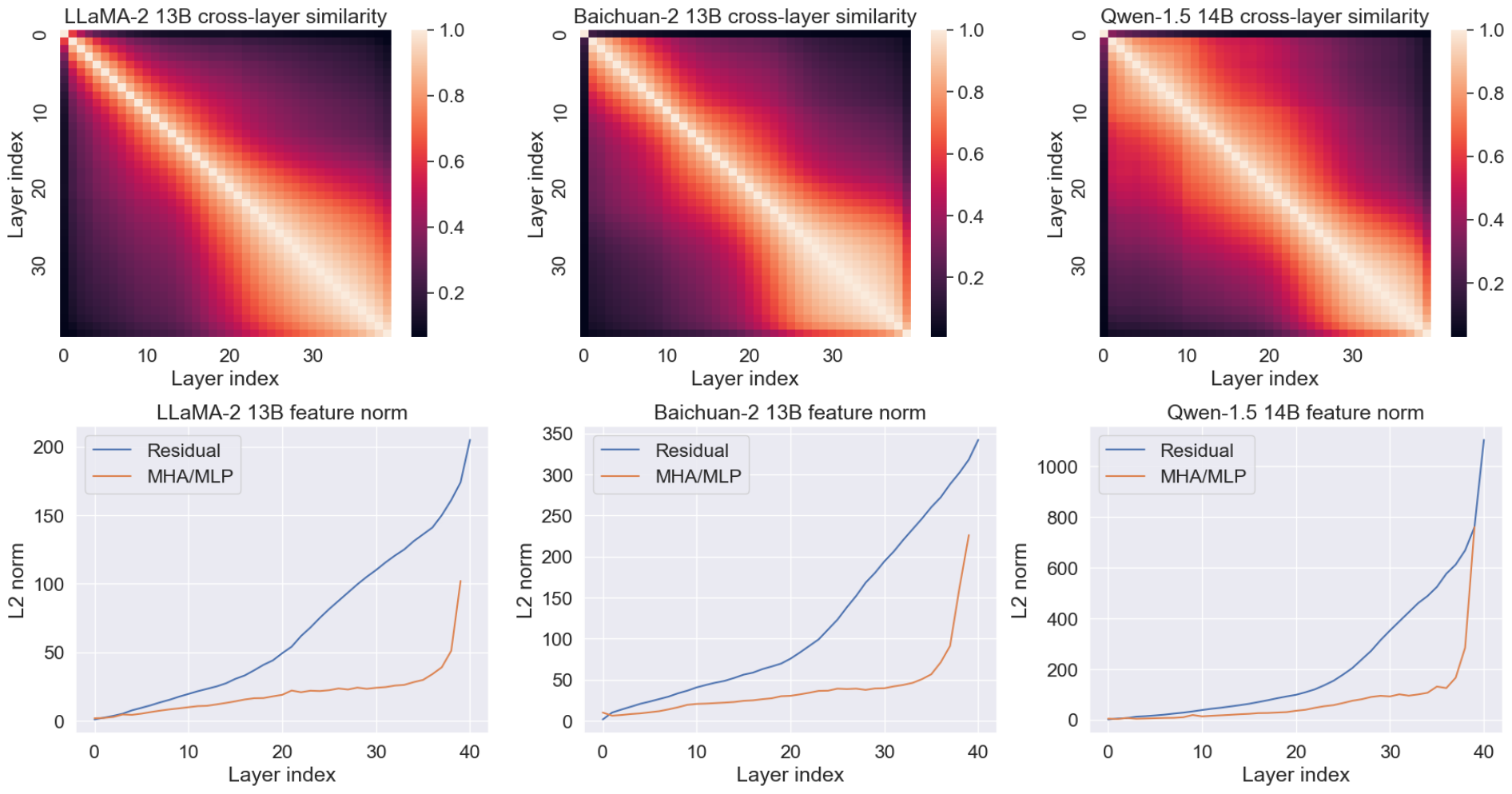}
    \caption{Cross-layer similarity and feature norm in multiple architectures.}
    \label{fig:more_empirical_results}
\end{figure}

\newpage
\section{Implementations}
\label{appendix:implementation}

\subsection{Modified algorithms for grouped query attention}
Some modern LLMs, such as LLaMA-3~\citep{llama_v3}, utilize a shared key-value strategy to improve inference efficiency, which is denoted as Grouped Query Attention (GQA).
To keep the pruned architecture the same as the original attention blocks, we modify the channel pruning on MHA.
Instead of finding the least important attention head individually, we find the least important pair of key and value.
Then we delete this pair and its corresponding queries.
We apply this modification to LLaMA-3 8B compression in the paper.

\subsection{Implementation details}

\paragraph{Setup} We utilize the HuggingFace generation library~\citep{huggingface-transformers} to implement our LLM models and use PyTorch~\citep{pytorch} Hooks for hidden states recording and correlation matrix estimations.
Unless otherwise specified, the experiments were conducted on 8 NVIDIA H800 80GB GPUs.
The models use the BF16 data format.
The calibration set consists of a random sample of 128 sequences, each of length 2048, from WikiText-2, following the common practice in the literature~\citep{slicegpt}.

\paragraph{Datasets} We consider multiple tasks in LM Evaluation Harness~\citep{eval-harness}, including ARC-e, ARC-c~\citep{arc}, PIQA~\citep{piqa}, WinoGrande~\citep{winogrande}, and HellaSwag~\citep{hellaswag}.

\paragraph{Correlation Matrix Estimations} Our algorithms utilize input correlation matrices in the MLP pruning method.
We gather the empirical data from the calibration set by registering the PyTorch~\citep{pytorch} hooks in the model.
Our process compresses the MLP blocks from all layers first, then compresses the MHA blocks from all layers.

\paragraph{Matrix Operations}
We utilize `torch.linalg.solve' in PyTorch for computing the inverse on tensors of dtype FP32.

\paragraph{MLP module} Our MLP pruning method requires a ridge leverage score parameter \(\lambda\).
We set \(\lambda\) to 10 times the mean singular value of the correlation matrix across all experiments.

\paragraph{MHA module}
Our MHA pruning method removes the entire group of query, key, and value matrices.

\paragraph{Recovery Fine-Tuning}
We use 50K samples of refined Alpaca for instruction tuning.
The models are trained with 2 epochs with a learning rate of \(1\times10^{-4}\).
The other primary hyperparameters used are \(lora\_alpha=16\), \(lora\_r=8\), and \(batch\_size=64\).

\paragraph{Latency and Throughput}
We test the throughput and latency results on token generation and prompt processing, respectively.
For token generation, we generate sentences with a length of 128 tokens and a batch size of 64.
For prompt processing, we measure the latency when processing an input sequence with 2048 tokens.
The HuggingFace~\cite{huggingface-transformers} implementation is used for benchmarking, and all experiments employ pipeline parallelism if distributed on multiple GPUS.

\subsection{Throughput and latency}
Table~\ref{tab:latency_throughtput_ppl} compares the latency, throughput, and perplexities of the compressed LLaMA-2~\citep{llama_v2} models with other pruning methods.
We test the throughput and latency results on token generation and prompt processing, respectively.
For token generation, we generate sentences with a length of 128 tokens and a batch size of 64.
For prompt processing, we measure the latency when processing an input sequence with 2048 tokens.
Even with dedicated hardware support, 2:4 pruning methods still lead to minor speedup (\(1.10\times\)) and lower throughput (\(0.98\times\) with a sparsity ratio of 50\%.
Width pruning methods, such as SliceGPT~\citep{slicegpt}, are more hardware-friendly and speed up the pruned model, while still lagging behind depth pruning methods.
FlattenGPT inherits the advantages of acceleration in depth pruning and further improves the performance.
Since the compressed model architecture of FlattenGPT is exactly the same as SLEB, the throughput and latency results are the same.
FlattenGPT outperforms all other methods in throughput (\(1.27\times\)) and latency (\(1.26\times\)), and achieves a comparable perplexities.
These results demonstrate that FlattenGPT has a better trade-off between speed and performance.

\begin{table}[ht]
\centering
\scriptsize
\caption{Throughput (tokens/s), latency (ms), and perplexity on WikiText-2 test split results. Throughput and latency are measured with LLaMA-2-70B on 2 NVIDIA A100 GPUs.}
\label{tab:latency_throughtput_ppl}
\renewcommand{\arraystretch}{0.9}
\setlength{\tabcolsep}{8pt}
\begin{tabular}{l|cc|cc|cc|ccc}
\toprule
\rowcolor{gray!20}
~	& \bf Pruning	& ~	& \bf Throughput	& ~	& \bf Latency	& ~	& \multicolumn{3}{c}{\bf LLaMA-2}	\\
\rowcolor{gray!20}
\multirow{-2}{*}{\bf Method}	& \bf Unit	& \multirow{-2}{*}{\bf Sparsity}	& \bf (Tokens/s)	& \multirow{-2}{*}{\bf Improve \(\uparrow\)}	& \bf (ms)	& \multirow{-2}{*}{\bf Speedup \(\uparrow\)}	& \bf 7B	& \bf 13B	& \bf 70B	\\
\midrule
\midrule
Dense		& -		& 0\%	& 299	& 1.00$\times$	& 1718.4	& 1.00$\times$	& 5.47	& 4.88	& 3.32	\\
\midrule
SparseGPT~\cite{sparsegpt}	& 2:4	& 50\%	& 293	& 0.98$\times$	& 1555.5	& 1.10$\times$	& 10.79	& 8.75	& 5.70	\\
Wanda~\cite{wanda}		& 2:4	& 50\%	& 293	& 0.98$\times$	& 1555.5	& 1.10$\times$	& 12.09	& 8.99	& 5.48	\\
DSnoT~\cite{DSnoT}		& 2:4	& 50\%	& 293	& 0.98$\times$	& 1555.5	& 1.10$\times$	& 11.97	& 8.87	& 5.49	\\
\midrule
LLM-Pruner~\cite{llm_pruner}	& Width & 20\%	& 314	& 1.05$\times$	& 1534.3	& 1.12$\times$	& 10.58	& 8.56	& -		\\
SliceGPT~\cite{slicegpt}	& Width & 20\%	& 314	& 1.05$\times$	& 1658.7	& 1.04$\times$	& 6.87	& 6.01	& 4.44	\\
SliceGPT~\cite{slicegpt}	& Width & 25\%	& 331	& 1.11$\times$	& 1440.7	& 1.19$\times$	& 7.55	& 6.63	& 4.89	\\
SliceGPT~\cite{slicegpt}	& Width & 30\%	& 343	& 1.15$\times$	& 1364.2	& 1.26$\times$	& 8.59	& 7.44	& 5.44	\\
\midrule
SLEB~\cite{sleb}		& Depth	& 20\%	& 381	& 1.27$\times$	& 1364.1	& 1.26$\times$	& 9.14	& 6.80	& 4.88 \\
\rowcolor{cyan!10}
FlattenGPT	& Depth	& 20\%	& 381	& 1.27$\times$	& 1364.1	& 1.26$\times$	& 8.68	& 6.50	& 4.79 \\
\bottomrule
\end{tabular}
\end{table}

\subsection{Pruning Computation Cost}
Table~\ref{tab:computation_cost} compares the compression times of FlattenGPT with the prevailing pruning methods, including SliceGPT~\citep{slicegpt}, LLM surgeon~\citep{llm_surgeon}, and MoDeGPT~\citep{modegpt}.
LLM Surgeon requires the gradient information of the LLMs, leading to heavy computation.
SliceGPT and ModeGPT do not leverage gradients, they can compress a model with fewer GPUs and computation time.
Our approach, FlattenGPT, is even faster than these methods, as we collect the correlation matrix of all layers at the same time.
Thus FlattenGPT is an efficient pruning method in this area.

\begin{table}[ht]
\scriptsize
\caption{Computation cost of pruning 20\% with FlattenGPT and recovery fine-tuning. The cost of the other methods are reported from the original paper, and thus the device is inconsistent, while it still illustrates the conclusion that our FlattenGPT is extremely fast. The calibration dataset consists of 128 samples with a sequence length of 2048.}
\label{tab:computation_cost}
\renewcommand{\arraystretch}{0.9}
\setlength{\tabcolsep}{12pt}
\centering
\begin{tabular}{l|c|cc|cc|c}
\toprule
\rowcolor{gray!20}
~								& ~						& \multicolumn{2}{c|}{\bf Pruning}	& \multicolumn{2}{c|}{\bf RFT}	&  \\
\rowcolor{gray!20}
\multirow{-2}{*}{\bf Method}	& \multirow{-2}{*}{\bf Model}	& \bf Time		& \bf GPUs				& \bf Time		& \bf GPUs			& \multirow{-2}{*}{\bf Total}	\\
\midrule
\multirow{2}{*}{SliceGPT~\cite{slicegpt}}		& LLaMA-2 7B	& 44m		& 1 H100 80GB		& 23m		& 1 H100 80GB	& 1h07m	\\
~								& LLaMA-2 13B	& 1h08m		& 1 H100 80GB		& 44m		& 1 H100 80GB	& 1h52m	\\
\midrule
\multirow{2}{*}{LLM surgeon~\cite{llm_surgeon}}	& LLaMA-2 7B	& 17h08m	& 4 H100 80GB		& -			& -				& -		\\
~								& LLaMA-2 13B	& 1d9h26m	& 8 H100 80GB		& -			& -				& -		\\
\midrule
\multirow{2}{*}{ModeGPT~\cite{modegpt}}		& LLaMA-2 7B	& 4h09m		& 1 A100 80GB		& 31m		& 1 A100 80GB	& 4h40m	\\
~								& LLaMA-2 13B	& 8h26m		& 1 A100 80GB		& -			& -				& -		\\
\midrule
\rowcolor{cyan!10}
~								& LLaMA-2 7B	& 7m		& 1 H800 80GB		& 25m		& 1 H800 80GB	& 32m	\\
\rowcolor{cyan!10}
\multirow{-2}{*}{FlattenGPT}	& LLaMA-2 13B	& 24m		& 1 H800 80GB		& 45m		& 1 H800 80GB	& 1h09m	\\
\bottomrule
\end{tabular}
\end{table}

\newpage
\section{Experiments}
\label{appendix:experiments}

\subsection{Additional comparison of Training-free pruning methods}
We compare the performance with other training-free pruning methods in Table~\ref{tab:zeroshot_training_free_all}, including both width compression and depth compression.
The width compression includes the 2:4 pruning methods SparseGPT~\citep{sparsegpt} and Wanda~\citep{wanda}, and structured channel pruning methods SliceGPT~\citep{slicegpt}.
The depth compression includes LaCo~\citep{laco}, SLEB~\citep{sleb}, Relative magnitude~\citep{rm}, ShortGPT~\citep{shortgpt}, and BlockPruner~\citep{blockpruner}.
FlattenGPT outperforms these methods on WikiText-2 perplexity and accuracy on the zero-shot downstream tasks, showcasing the effectiveness of our method.

\begin{table*}[h]
\scriptsize
\caption{Comparison with training-free pruning methods on WikiText-2 perplexity and accuracies on zero-shot tasks.}
\label{tab:zeroshot_training_free_all}
\renewcommand{\arraystretch}{0.9}
\setlength{\tabcolsep}{8pt}
\centering
\begin{tabular}{c | l | c | c | c  c  c  c  c | c}
\toprule
\rowcolor{gray!20}
\multicolumn{2}{c|}{\bf Method} & \bf Sparsity	& \bf PPL \(\downarrow\)	& \bf WinoG	& \bf HellaS	& \bf PIQA	& \bf ARC-e	& \bf ARC-c	& \bf Avg.  \\
\midrule
\midrule
\multicolumn{2}{c|}{LLaMA-2 7B (original)}	& 0\%		& 5.47	& 69.06	& 75.99	& 79.11	& 74.58	& 46.25	& 69.00	\\
\midrule
\multirow{3}{*}{Width}	& SpareGPT~\citep{sparsegpt}	& 2:4 (50\%)& 10.79	& 64.96	& 58.93	& 72.14	& 60.90	& 34.22	& 58.23	\\
& Wanda~\citep{wanda}						& 2:4 (50\%)& 12.09	& 62.27	& 55.33	& 70.84	& 57.58	& 31.91	& 55.59	\\
& SliceGPT~\citep{slicegpt}					& 21.45\%	& 7.02	& 59.91	& 56.04	& 72.42	& 63.64	& 37.12	& 57.83	\\
\midrule
\multirow{6}{*}{Depth}	& SLEB~\citep{sleb}	& 21.02\%	& 9.14	& 58.96	& 62.47	& 73.07	& 56.48	& 33.02	& 56.80	\\
	& LaCo~\citep{laco}						& 21.02\%	& 50.39	& 60.46	& 54.08	& 68.34	& 55.39	& 35.84	& 54.82	\\
	& RM~\citep{rm}							& 21.02\%	& 676.8	& 49.25	& 29.22	& 54.46	& 34.43	& 22.53	& 37.98	\\
	& ShortGPT~\citep{shortgpt}				& 21.02\%	& 18.45	& 65.90	& 62.63	& 70.24	& 56.06	& 36.09	& 58.18	\\
	& BlockPruner~\citep{blockpruner}		& 21.99\%	& 11.51	& 62.43	& 65.87	& 74.21	& 61.07	& 37.29	& 60.17	\\
\rowcolor{cyan!10}
\cellcolor{white}		& FlattenGPT		& 21.02\%	& \bf 8.68	& 66.54	& 68.45	& 72.74	& 63.43	& 41.30	& \bf 62.49	\\
\midrule
\midrule
\multicolumn{2}{c|}{LLaMA-2 13B (original)}	& 0\%		& 4.88	& 72.22	& 79.39	& 80.47	& 77.48	& 49.23	& 71.76	\\
\midrule
\multirow{3}{*}{Width}	& SpareGPT~\citep{sparsegpt}	& 2:4 (50\%)& 8.75	& 68.51	& 65.52	& 75.46	& 66.04	& 39.76	& 63.06	\\
	& Wanda~\citep{wanda}					& 2:4 (50\%)& 8.99	& 67.01	& 63.09	& 73.94	& 64.31	& 37.80	& 61.23	\\
	& SliceGPT~\citep{slicegpt}				& 25\%		& 6.63	& 67.48	& 58.10	& 68.55	& 62.50	& 37.88	& 58.90	\\
\midrule
\multirow{5}{*}{Depth}	& LaCo~\citep{laco}	& 24.37\%	& 13.97	& 59.27	& 60.44	& 72.42	& 54.34	& 34.56	& 56.21	\\
	& RM~\citep{rm}							& 24.37\%	& 10.08	& 66.61	& 66.80	& 73.72	& 66.12	& 41.98	& 63.05	\\
	& ShortGPT~\citep{shortgpt}				& 24.37\%	& 20.06	& 70.80	& 67.80	& 72.74	& 60.35	& 41.30	& 62.60	\\
	& BlockPruner~\citep{blockpruner}		& 25.12\%	& 8.16	& 66.30	& 72.20	& 76.93	& 65.82	& 41.38	& 64.53	\\
\rowcolor{cyan!10}
\cellcolor{white}		& FlattenGPT		& 24.37\%	& \bf 6.68	& 71.11	& 73.44	& 76.33	& 72.10	& 44.54	& \bf 67.50	\\
\midrule
\midrule
\multicolumn{2}{c|}{LLaMA-2 70B (original)}	& 0\%		& 3.32	& 77.98	& 83.84	& 82.70	& 80.98	& 57.34	& 76.57	\\
\midrule
\multirow{3}{*}{Width}	& SpareGPT~\citep{sparsegpt}	& 2:4 (50\%)& 5.70	& 76.56	& 76.09	& 80.03	& 76.94	& 49.74	& 71.87	\\
	& Wanda~\cite{wanda}					& 2:4 (50\%)& 5.48	& 74.66	& 79.22	& 80.30	& 76.35	& 51.19	& 72.34	\\
	& SliceGPT~\cite{slicegpt}				& 20\%		& 4.44	& 74.92	& 72.98	& 76.61	& 80.51	& 55.20	& 72.34	\\
\midrule
\multirow{3}{*}{Depth}	& SLEB~\cite{sleb}	& 19.84\%	& 4.88	& 72.93	& 77.21	& 80.14	& 75.38	& 48.38	& 70.81	\\
& ShortGPT~\cite{shortgpt}					& 19.84\%	& 66.33	& 71.96	& 78.87	& 76.02	& 76.02	& 52.95	& 71.68	\\
\rowcolor{cyan!10}
\cellcolor{white}	& FlattenGPT			& 19.84\%	& \bf 4.79	& 77.35	& 81.42	& 80.36	& 77.48	& 53.07 & \bf 73.94	\\
\midrule
\midrule
\multicolumn{2}{c|}{Baichuan-2 7B (original)}& 0\%		& 6.04	& 68.27	& 72.18	& 77.48	& 72.98	& 42.75	& 66.73	\\
\midrule
\multirow{5}{*}{Depth}	& LaCo~\citep{laco}	& 21.57\%	& 26.46	& 58.56	& 51.50	& 68.28	& 52.90	& 28.50	& 51.95	\\
	& RM~\citep{rm}							& 21.57\%	& 189.8	& 52.33	& 30.87	& 59.96	& 38.17	& 23.63	& 40.99	\\
	& ShortGPT~\citep{shortgpt}				& 21.57\%	& 31.05	& 62.67	& 50.01	& 63.71	& 47.31	& 30.72	& 50.88	\\
	& BlockPruner~\citep{blockpruner}		& 22.45\%	& \bf 15.38	& 61.48	& 58.09	& 69.75	& 58.08	& 33.02	& 56.08	\\
\rowcolor{cyan!10}
\cellcolor{white}		& FlattenGPT		& 21.57\%	& 20.55	& 64.33	& 61.50	& 69.42	& 56.27	& 35.24	& \bf 57.35	\\
\midrule
\midrule
\multicolumn{2}{c|}{Baichuan-2 13B (original)}& 0\%		& 6.66	& 70.40	& 75.23	& 78.84	& 74.07	& 47.70	& 69.25	\\
\midrule
\multirow{5}{*}{Depth}	& LaCo~\citep{laco}	& 22.68\%	& 27.07	& 58.01	& 54.00	& 70.89	& 57.11	& 32.94	& 54.59	\\
	& RM~\citep{rm}							& 22.68\%	& 17.70	& 67.88	& 63.78	& 68.99	& 57.49	& 37.54	& 59.14	\\
	& ShortGPT~\citep{shortgpt}				& 22.68\%	& 20.69	& 68.27	& 61.71	& 69.31	& 56.52	& 36.69	& 58.50	\\
	& BlockPruner~\citep{blockpruner}		& 24.19\%	& 15.36	& 64.01	& 64.20	& 71.44	& 59.81	& 37.88	& 59.47	\\
\rowcolor{cyan!10}
\cellcolor{white}		& FlattenGPT		& 22.68\%	& \bf 13.71	& 68.19	& 65.27	& 71.22	& 58.75	& 37.03	& \bf 60.09	\\
\midrule
\midrule
\multicolumn{2}{c|}{Qwen-1.5 7B (original)}	& 0\%		& 7.95	& 66.46	& 76.92	& 79.22	& 62.16	& 42.66	& 65.48	\\
\midrule
\multirow{5}{*}{Depth}	& LaCo~\citep{laco}	& 20.97\%	& 39.23	& 58.64	& 56.35	& 70.40	& 46.89	& 32.85	& 53.03	\\
	& RM~\citep{rm}							& 20.97\%	& 2026	& 49.88	& 42.00	& 67.36	& 54.17	& 28.58	& 48.40	\\
	& ShortGPT~\citep{shortgpt}				& 20.97\%	& 49.88	& 62.12	& 58.87	& 69.53	& 43.60	& 32.17	& 53.26	\\
	& BlockPruner~\citep{blockpruner}		& 21.83\%	& 20.58	& 55.56	& 59.31	& 71.71	& 53.70	& 33.28	& 54.71	\\
\rowcolor{cyan!10}
\cellcolor{white}		& FlattenGPT		& 20.97\%	& \bf 16.05	& 59.27	& 62.89	& 68.39	& 56.99	& 37.46	& \bf 57.00	\\
\midrule
\midrule
\multicolumn{2}{c|}{Qwen-1.5 14B (original)}	& 0\%		& 7.44	& 70.56	& 79.41	& 79.87	& 68.48	& 47.01	& 69.07	\\
\midrule
\multirow{5}{*}{Depth}	& LaCo~\citep{laco}	& 22.25\%	& 16.32	& 58.33	& 60.16	& 71.55	& 53.70	& 34.04	& 55.56	\\
	& RM~\citep{rm}							& 22.25\%	& 55.99	& 53.28	& 42.08	& 67.08	& 50.72	& 29.01	& 48.43	\\
	& ShortGPT~\citep{shortgpt}				& 22.25\%	& 1237	& 55.96	& 36.16	& 58.60	& 38.09	& 34.81	& 44.72	\\
	& BlockPruner~\citep{blockpruner}		& 23.72\%	& 15.67	& 61.48	& 66.92	& 75.24	& 59.51	& 39.08	& 60.45	\\
\rowcolor{cyan!10}
\cellcolor{white}		& FlattenGPT		& 22.25\%	& \bf 11.55	& 65.59	& 68.57	& 74.10	& 65.03	& 40.78	& \bf 62.81	\\
\bottomrule
\end{tabular}
\end{table*}

\subsection{Additional comparison of recovery fine-tuning}
Table~\ref{tab:zeroshot_rft_all} shows the impact of Recovery Fine-Tuning (RFT).
Our method outperforms previous methods after RFT.
This is because the flattening method retains the knowledge from all layers and makes it easier for fine-tuning.

\begin{table}[h]
\scriptsize
\caption{Zero-shot task performance of recovery fine-tuning. $^\dagger$ indicates fine-tuned on Alpaca~\citep{alpaca} dataset.}
\label{tab:zeroshot_rft_all}
\renewcommand{\arraystretch}{0.9}
\setlength{\tabcolsep}{10pt}
\centering
\begin{tabular}{c | l | c | c  c  c  c  c | c}
\toprule
\rowcolor{gray!20}
\multicolumn{2}{c|}{\bf Method}	& \bf Sparsity	& \bf WinoG & \bf HellaS	& \bf PIQA	& \bf ARC-e	& \bf ARC-c	& \bf Avg.  \\
\midrule
\midrule
\multicolumn{2}{c|}{LLaMA-2 7B (original)}			& 0\%		& 69.06	& 75.99	& 79.11	& 74.58	& 46.25	& 69.00	\\
\midrule
\multirow{3}{*}{Width}	& Wanda-sp~\citep{wanda}	& 18.81\%	& 63.77	& 70.66	& 76.44	& 69.61	& 42.15	& 64.53	\\
& FLAP~\citep{flap}									& 19.19\%	& 64.72	& 64.69	& 73.39	& 62.25	& 32.51	& 59.51	\\
& LLM-Pruner~\citep{llm_pruner}						& 18.82\%	& 61.17	& 66.13	& 76.66	& 64.86	& 37.88	& 61.34	\\
& LLM-Pruner$^\dagger$~\citep{llm_pruner}			& 18.82\%	& 61.88	& 67.13	& 77.48	& 65.78	& 38.48	& 62.15	\\
\midrule
\multirow{7}{*}{Depth}	& SLEB~\citep{sleb}			& 18.02\%	& 59.75	& 63.95	& 73.94	& 63.47	& 35.15	& 59.25	\\
& Shortened LLaMA~\citep{shortenedllama} 			& 18.02\%	& 57.46	& 63.36	& 73.78	& 64.02	& 33.19	& 58.36	\\
& Shortened LLaMA$^\dagger$~\citep{shortenedllama} 	& 18.02\%	& 58.80	& 67.99	& 76.06	& 68.81	& 37.88	& 61.91	\\
& SLM~\citep{slm}									& 18.02\%	& 66.30	& 65.10	& 70.24	& 61.45	& 38.31	& 60.28	\\
& SLM$^\dagger$~\citep{slm}							& 18.02\%	& 67.09	& 70.48	& 73.67	& 69.11	& 41.21	& 64.31	\\
\rowcolor{cyan!10}
\cellcolor{white}	& FlattenGPT					& 18.02\%	& 67.40	& 70.74	& 74.59	& 64.44	& 41.98	& 63.83	\\
\rowcolor{cyan!10}
\cellcolor{white}	& FlattenGPT$^\dagger$			& 18.02\%	& 68.75	& 73.01	& 74.97	& 67.40	& 45.05	& 66.24	\\
\midrule
\midrule
\multicolumn{2}{c|}{LLaMA-2 13B (original)}			& 0\%		& 72.22	& 79.39	& 80.47	& 77.48	& 49.23	& 71.76	\\
\midrule
\multirow{3}{*}{Width}	& Wanda-sp~\citep{wanda}	& 19.49\%	& 67.01	& 74.75	& 77.48	& 73.48	& 44.11	& 67.37	\\
& FLAP~\citep{flap}									& 19.47\%	& 68.35	& 69.07	& 74.65	& 70.83	& 40.61	& 64.70	\\
& LLM-Pruner~\citep{llm_pruner}						& 19.48\%	& 64.17	& 72.02	& 78.51	& 69.99	& 43.60	& 65.66	\\
& LLM-Pruner$^\dagger$~\citep{llm_pruner}			& 19.48\%	& 67.32	& 74.84	& 79.16	& 73.49	& 43.77	& 67.72	\\
\midrule
\multirow{7}{*}{Depth}	& SLEB~\citep{sleb}			& 19.50\%	& 64.96	& 70.55	& 76.61	& 64.35	& 38.31	& 62.96	\\
& Shortened LLaMA~\citep{shortenedllama}			& 19.50\%	& 70.48	& 71.19	& 75.03	& 69.53	& 43.09	& 65.86	\\
& Shortened LLaMA$^\dagger$~\citep{shortenedllama}	& 19.50\%	& 71.11	& 75.20	& 76.28	& 74.79	& 46.67	& 68.81	\\
& SLM~\citep{slm}									& 19.50\%	& 70.80	& 67.73	& 72.36	& 64.82	& 39.68	& 63.08	\\
& SLM$^\dagger$~\citep{slm}							& 19.50\%	& 71.67	& 76.37	& 77.42	& 76.56	& 48.55	& 70.11	\\
\rowcolor{cyan!10}
\cellcolor{white}	& FlattenGPT					& 19.50\%	& 71.43	& 75.26	& 77.58	& 71.68	& 45.39	& 68.27	\\
\rowcolor{cyan!10}
\cellcolor{white}	& FlattenGPT$^\dagger$			& 19.50\%	& 71.82	& 77.85	& 78.73	& 75.08	& 49.15	& 70.53	\\
\midrule
\midrule
\multicolumn{2}{c|}{LLaMA-3 8B (original)}			& 0\%		& 73.40	& 79.17	& 79.49	& 80.09	& 53.24	& 73.08	\\
\midrule
\multirow{2}{*}{Width}	& FLAP~\citep{flap}			& 16.30\%	& 49.96	& 26.36	& 52.18	& 26.81	& 24.83	& 36.03	\\
& LLM-Pruner~\citep{llm_pruner}						& 15.39\%	& 68.67	& 67.79	& 77.04	& 68.60	& 39.08	& 64.23	\\
& LLM-Pruner$^\dagger$~\citep{llm_pruner}			& 15.39\%	& 70.32	& 74.27	& 79.49	& 74.29	& 46.59	& 68.99	\\
\midrule
\multirow{6}{*}{Depth}	& Shortened LLaMA~\citep{shortenedllama} 	& 16.30\%	& 57.85	& 60.99	& 73.23	& 65.40	& 34.04	& 58.30	\\
& Shortened LLaMA$^\dagger$~\citep{shortenedllama} 	& 16.30\%	& 62.75	& 72.70	& 78.07	& 75.30	& 44.80	& 66.72	\\
& SLM~\citep{slm}									& 16.30\%	& 69.61	& 61.8	& 71.98	& 66.04	& 41.81	& 62.25	\\
& SLM$^\dagger$~\citep{slm}							& 16.30\%	& 71.74	& 73.77	& 77.64	& 76.60	& 50.94	& 70.14	\\
\rowcolor{cyan!10}
\cellcolor{white}	& FlattenGPT					& 16.30\%	& 71.82	& 70.63	& 72.91	& 69.1	& 46.59	& 66.21	\\
\rowcolor{cyan!10}
\cellcolor{white}	& FlattenGPT$^\dagger$			& 16.30\%	& 73.09	& 75.93	& 77.09	& 75.72	& 50.34	& 70.43	\\
\bottomrule
\end{tabular}
\end{table}

\subsection{Effectiveness of Flattening}

\noindent\textbf{Flattened layer indices:}
We show which transformer blocks are chosen to be flattened in Figure~\ref{fig:flattened_layer_indices}.
The location of flattened transformer blocks is highly consistent across various target models.
The late blocks are almost flattened except the last one or two, whereas the early blocks are rarely selected.
This is related to the similarity distribution in the model, where the late blocks have more similar input features.

\noindent\textbf{Performance after flattening:}
We need to answer the question: \textit{How does flattening improve the performance of the depth-compressed model?}
The answer is that \noindent\textbf{Flattening preserves more knowledge}.
Compared with the layer pruning methods, flattening preserves the parameters and thus preserves the knowledge in the parameters.
This knowledge facilitates performance maintenance during depth compression.
Figure~\ref{fig:prune_flatten_comparison} illustrates the comparison of layer pruning and layer flattening on LLaMA-2 7B.
We use the same layer index in both settings, \emph{i.e.}, to prune the selected layer or merge the selected layer with the prior layer.
In the flattening experiments, the model performance gradually drops as the number of flattened layers increases.
After flattening 8 layers, it has maintained 98\% of accuracy on zero-shot tasks and has a 19\% degradation on perplexity.
This result leaves plenty of room for channel pruning.
However, on the contrary, layer pruning quickly loses performance with merely one or two pruned layers.
It only maintains 80\% of accuracy on zero-shot tasks and 319\% degradation on perplexity!
With such information loss, layer-pruning-based methods are very limited and cannot achieve high performance.
Our flattening method has alleviated this problem, thus providing an effective way of depth compression.




\begin{figure}[ht]
\centering
\begin{minipage}[t]{0.45\textwidth}
    \centering
    \includegraphics[width=\linewidth]{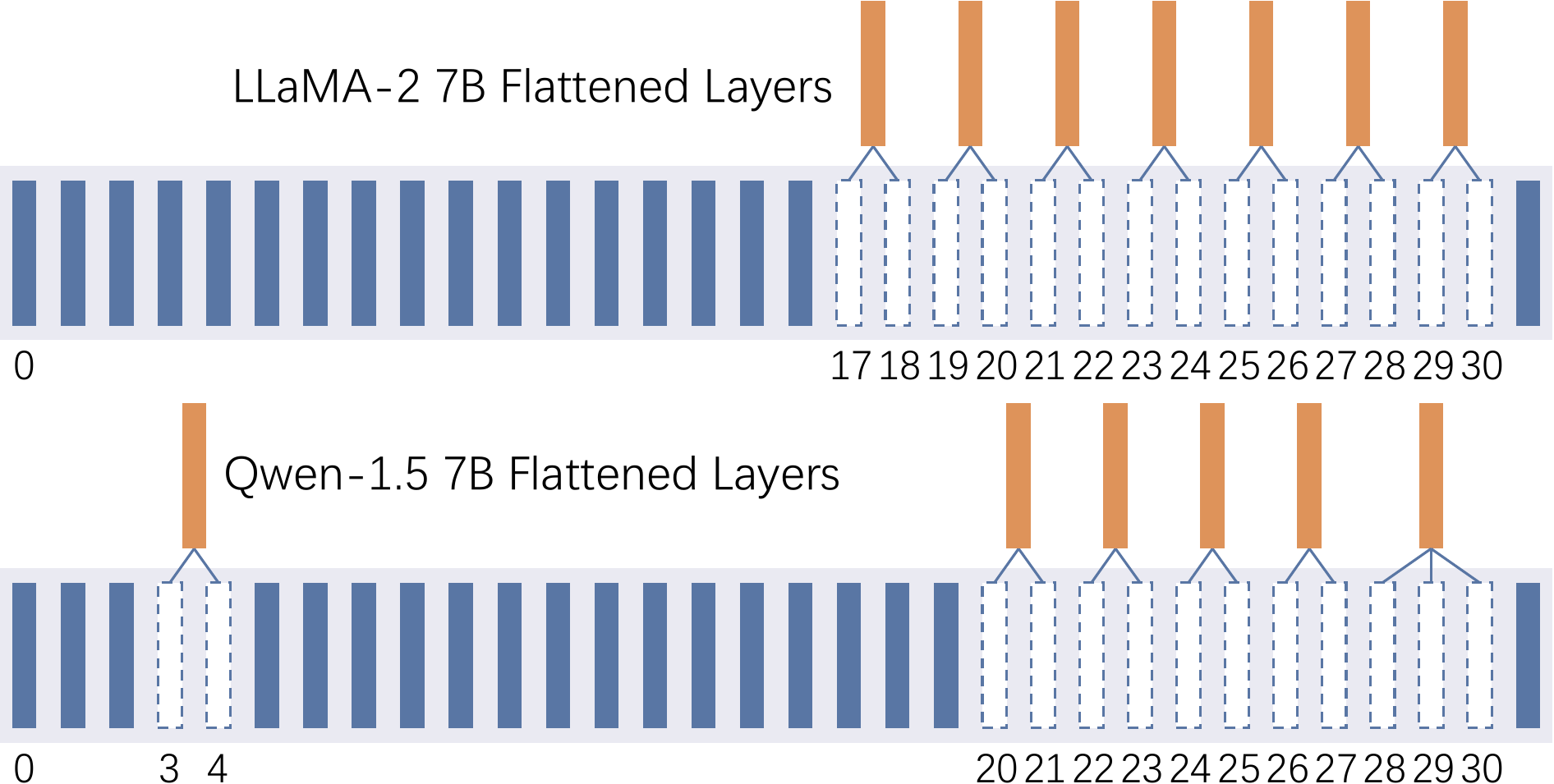}
    \caption{Flattened Layer indices.}
    \label{fig:flattened_layer_indices}
\end{minipage}
\hfill
\begin{minipage}[t]{0.53\textwidth}
\centering
  \begin{minipage}[t]{0.49\textwidth}
    \centering
    \includegraphics[width=\linewidth]{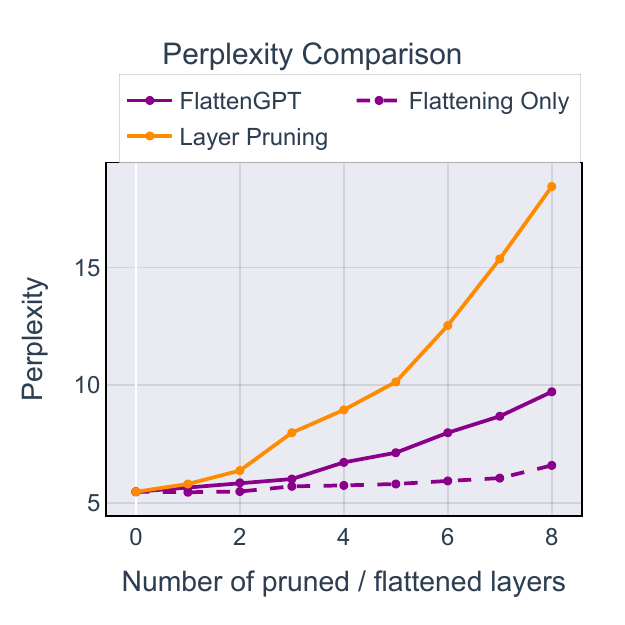}
  \end{minipage}
  \hfill
  \begin{minipage}[t]{0.49\textwidth}
    \centering
    \includegraphics[width=\linewidth]{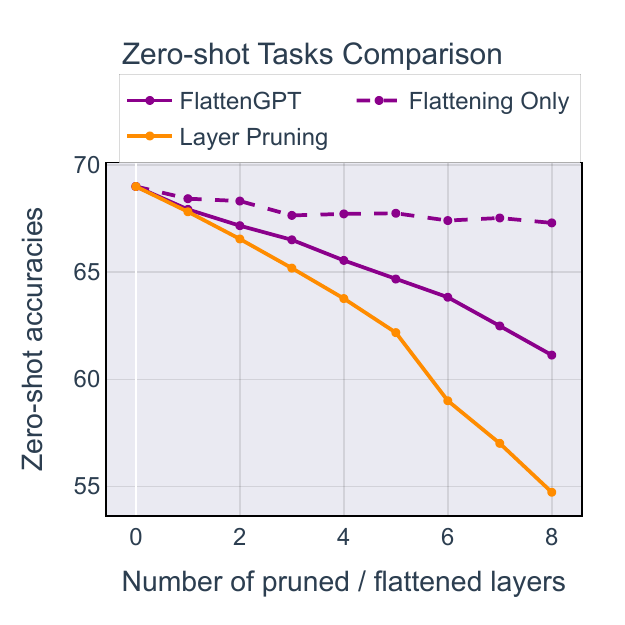}
  \end{minipage}
  \caption{Comparison of layer pruning and flattening.}
  \label{fig:prune_flatten_comparison}
\end{minipage}
\end{figure}


\subsection{Effectiveness of our channel pruning method}
The flattening operation changes the depth compression task into a channel pruning task.
This method shows an advantage of fine-grained depth compression, whereas it relies on the performance of the channel pruning method.
In this paper, we use a simple yet effective channel pruning method.
To validate the effectiveness of our channel pruning method, we conduct experiments with channel pruning only.
We use the sparsity distribution described in ModeGPT~\citep{modegpt}, and compare the channel pruning performance with other channel pruning methods.
As shown in Table~\ref{tab:channel_pruning}, our channel pruning approach has a clear advantage over previous pruning methods.
By combining the MHA pruning and MLP pruning, our method achieves the best performance, surpassing the previous channel pruning method, including SliceGPT~\citep{slicegpt} and ModeGPT~\citep{modegpt}.

\begin{table}[h]
\scriptsize
\caption{Zero-shot task performance of channel pruning methods calibrated with 128 samples from WikiText-2.}
\label{tab:channel_pruning}
\renewcommand{\arraystretch}{0.9}
\setlength{\tabcolsep}{11pt}
\centering
\begin{tabular}{l | c | c  c  c  c  c | c}
\toprule
\rowcolor{gray!20}
\bf Method	& \bf Sparsity	& \bf WinoG & \bf HellaS	& \bf PIQA	& \bf ARC-e	& \bf ARC-c	& \bf Avg.  \\
\midrule
LLaMA-2 7B (original)	& 0\%		& 69.06	& 75.99	& 79.11	& 74.58	& 46.25	& 69.00	\\
\midrule
SliceGPT				& 20\%		& 62.74	& 49.78	& 64.25	& 51.47	& 31.06	& 51.86	\\
ModeGPT					& 20\%		& 68.03	& 69.05	& 74.05	& 69.07	& 42.06	& 64.46	\\
\midrule
Our MLP pruning			& 20\%		& 66.06	& 66.54	& 73.23	& 65.19	& 38.91	& 61.99	\\
Our MHA pruning			& 21.02\%	& 66.93	& 69.64	& 73.94	& 63.97	& 42.24	& 63.34	\\
\rowcolor{cyan!10}
Our MHA + MLP Pruning		& 21.07\%	& 68.03	& 71.64	& 76.17	& 68.98	& 44.28	& \bf 65.82	\\
\midrule
\midrule
LLaMA-2 13B (original)	& 0\%		& 72.22	& 79.39	& 80.47	& 77.48	& 49.23	& 71.76	\\
\midrule
SliceGPT				& 20\%		& 67.17	& 53.58	& 65.83	& 55.81	& 35.84	& 55.65	\\
ModeGPT					& 20\%		& 70.32	& 68.96	& 74.53	& 74.07	& 46.16	& 66.81	\\
\midrule
\rowcolor{cyan!10}
Our MHA + MLP Pruning		& 21.07\%	& 71.43	& 75.26	& 77.58	& 71.68	& 45.39	& \bf 68.94	\\
\bottomrule
\end{tabular}
\end{table}

We further make ablations on the effectiveness of Nystr\"om approximation.
As shown in Table~\ref{tab:ablation_nystrom}, Nystr\"om approximation outperforms the channel selection only method, demonstrating the effectiveness of adjusting the down projection.
\begin{table}[h]
\scriptsize
\caption{Comparison of channel selection and Nystr\"om approximation.}
\centering
\renewcommand{\arraystretch}{0.9}
\setlength{\tabcolsep}{11pt}
\begin{tabular}{l|ccccccc}
\toprule
\bf Method	& \bf Sparsity	& \bf WinoG & \bf HellaS	& \bf PIQA	& \bf ARC-e	& \bf ARC-c	& \bf Avg.  \\
\midrule
Channel Selection			& 20\%	& 66.46	& 65.48	& 71.22	& 63.13	& 39.93	& 61.24	\\
\rowcolor{cyan!10}
+ Nystr\"om approximation	& 20\%	& 66.54	& 68.45	& 72.74	& 63.43	& 41.30	& 62.49	\\
\bottomrule
\end{tabular}
\label{tab:ablation_nystrom}
\end{table}

\subsection{Advantages of depth compression over width compression}
In this paper, we focus on the depth compression tasks.
Although previous depth compression methods perform much worse than the width compression ones, FlattenGPT has built a novel approach to improve this performance greatly.
In the main paper, we have shown that FlattenGPT achieves a better trade-off between performance and speed.
In this part, we will further show that FlattenGPT shows promising performance compared with the latest width compression method after recovery fine-tuning.
Table~\ref{tab:channel_pruning_rft} shows the performance with or without RFT.
LLM-pruner~\citep{llm_pruner} shows little improvement with RFT.
LLM Surgeon~\citep{llm_surgeon} does not show the results, but it claimed that LoRA improves compression performance in the smallest OPT-125m model, but not in larger models.
ModeGPT~\citep{modegpt} even demonstrates performance loss after RFT, which illustrates that the model probably suffers from overfitting.
FlattenGPT unifies the two tasks of deep compression and channel compression, making the pruned model more suitable for fine-tuning.
This is more practical than previous pruning methods.

\begin{table}[h]
\scriptsize
\caption{Zero-shot task performance of channel pruning methods calibrated with 128 samples from WikiText-2. $^\dagger$ indicates fine-tuned on Alpaca~\citep{alpaca} dataset. ModeGPT~\citep{modegpt} employs Alpaca as the calibration dataset. LLM Surgeon~\citep{llm_surgeon} does not show the results but claims that LoRA cannot improve the performance.}
\label{tab:channel_pruning_rft}
\renewcommand{\arraystretch}{0.9}
\setlength{\tabcolsep}{11pt}
\centering
\begin{tabular}{l | c | c  c  c  c  c | c}
\toprule
\rowcolor{gray!20}
\bf Method	& \bf Sparsity	& \bf WinoG & \bf HellaS	& \bf PIQA	& \bf ARC-e	& \bf ARC-c	& \bf Avg.  \\
\midrule
LLaMA-2 7B (original)	& 0\%		& 69.06	& 75.99	& 79.11	& 74.58	& 46.25	& 69.00	\\
\midrule
LLM-Pruner				& 18.82\%	& 61.17	& 66.13	& 76.66	& 64.86	& 37.88	& 61.34	\\
LLM-Pruner$^\dagger$	& 18.82\%	& 61.88	& 67.13	& 77.48	& 65.78	& 38.48	& 62.15	\\
LLM Surgeon				& 20\%		& 66.30	& 71.30	& 77.09	& 71.36	& 41.89	& 65.59	\\
ModeGPT					& 20\%		& 68.19	& 69.59	& 76.22	& 71.71	& 41.89	& 65.52	\\
ModeGPT$^\dagger$		& 20\%		& 66.30	& 68.07	& 77.20	& 70.45	& 42.92	& 64.99	\\
\midrule
\rowcolor{cyan!10}
FlattenGPT				& 20\%	& 67.40	& 70.74	& 74.59	& 64.44	& 41.98	& 63.83	\\
\rowcolor{cyan!10}
FlattenGPT$^\dagger$	& 20\%	& 68.75	& 73.01	& 74.97	& 67.40	& 45.05	& \bf 66.24	\\
\bottomrule
\end{tabular}
\end{table}

\subsection{Locations of flattened layers}
We show which transformer blocks are chosen to be flattened in Table~\ref{tab:flattened_layer_indices}.
The location of flattened transformer blocks is highly consistent across various target models.
The late blocks are almost flattened, except the last one or two, whereas the early blocks are rarely selected.
This is related to the similarity distribution in the model, where the late blocks have more similar input features.

\begin{table}[h]
\scriptsize
\caption{Locations of flattened Transformer blocks with target sparsity of 20\%.}
\label{tab:flattened_layer_indices}
\renewcommand{\arraystretch}{1.0}
\setlength{\tabcolsep}{25pt}
\centering
\begin{tabular}{l|c}
\toprule
\rowcolor{gray!20}
\bf Models		& \bf Merged Layer Index \\
\midrule
LLaMA-2 7B	& [[17, 18], [19, 20], [21, 22], [23, 24], [25, 26], [27, 28], [29, 30]] \\
\midrule
LLaMA-2 13B & [[23, 24], [25, 26], [27, 28], [29, 30], [31, 32], [33, 34], [35, 36], [37, 38]]\\
\midrule
\multirow{2}{*}{LLaMA-2 70B}	& [[14, 15], [46, 47], [49, 50], [51, 52], [54, 55], [57, 58], [59, 60, 61], \\
~			& [62, 63, 64], [65, 66, 67], [68, 69], [70, 71], [72, 73], [74, 75]] \\
\midrule
LLaMA-3 8B	& [[16, 17], [18, 19], [20, 21], [23, 24], [25, 26], [27, 28], [29, 30]] \\
\midrule
Qwen-1.5 7B	& [[3, 4], [20, 21], [22, 23], [24, 25], [26, 27], [28, 29, 30]] \\
\midrule
Qwen-1.5 14B& [[7, 8], [10, 11], [19, 20], [24, 25], [26, 27], [28, 29], [30, 31], [32, 33], [34, 35], [36, 37]]\\
\bottomrule
\end{tabular}
\end{table}

\subsection{Dependency on calibration dataset}
We evaluate the dependency on the calibration dataset in Table~\ref{tab:calibration_dataset}.
We use the calibration set size of 128 and sequence length of 2048 for WikiText-2~\citep{wikitext2} and Alpaca datasets.
The results show that WikiText-2 has a slightly better performance, probably due to the dataset quality.
The alpaca dataset is not as representative as a high-quality dataset, thus the performance is slighter lower than WikiText-2.

\begin{table}[h]
\scriptsize
\caption{Results on different calibration dataset.}
\label{tab:calibration_dataset}
\renewcommand{\arraystretch}{0.9}
\setlength{\tabcolsep}{6pt}
\centering
\begin{tabular}{l | c | c | c | c  c  c  c  c | c}
\toprule
\rowcolor{gray!20}
\bf Method	& \bf Dataset	& \bf PPL	& \bf Sparsity	& \bf WinoG & \bf HellaS	& \bf PIQA	& \bf ARC-e	& \bf ARC-c	& \bf Avg.  \\
\midrule
LLaMA-2 7B (original)	& -				& 0\%		& 5.47	& 69.06	& 75.99	& 79.11	& 74.58	& 46.25	& 69.00	\\
\midrule
FlattenGPT				& WikiText-2	& 21.02\%	& 8.68	& 67.40	& 70.74	& 74.59	& 64.44	& 41.98	& 63.83	\\
FlattenGPT				& Alpaca		& 21.02\%	& 11.84	& 67.64	& 67.92	& 72.31	& 62.54	& 39.25	& 61.93	\\
\bottomrule
\end{tabular}
\end{table}

\subsection{Dependency on the Calibration Dataset Size}
We test the size of the calibration dataset from 64 to 1024 samples as shown in Table~\ref{tab:calibration_dataset_size}. Results confirm that 128 samples suffice, as larger sets yield marginal gains ($<0.2\%$).

\begin{table}[h]
\scriptsize
\caption{The zero-shot accuracies on LLaMA-2 7B with different calibration dataset size.}
\label{tab:calibration_dataset_size}
\renewcommand{\arraystretch}{0.9}
\setlength{\tabcolsep}{10pt}
\centering
\begin{tabular}{l | c  c  c  c  c}
\toprule
\rowcolor{gray!20}
\bf Num of Samples	& \bf 64	& \bf 128	& \bf 256	& \bf 512	& \bf 1024	\\
\midrule
Accuracy			& 61.17		& 62.49		& 62.25		& 62.58		& 62.60		\\
\bottomrule
\end{tabular}
\end{table}

\subsection{Generalization on other tasks}

We conduct experiments on InternVL-C 6B, which is a large vision transformer that exhibits a similar cross-layer similarity pattern to the LLMs.
The results in Table~\ref{tab:generalization_vit} show that our method has good generalization ability on vision transformers.
The multimodal transformers are usually composed of an LLM transformer and a vision encoder transformer.
Therefore, it is reasonable to apply our method to the LLM and the vision encoder individually.
\begin{table}[h]
\scriptsize
\caption{The zero-shot accuracies on InternVL-C 6B with 20\% pruning ratio.}
\label{tab:generalization_vit}
\renewcommand{\arraystretch}{0.9}
\setlength{\tabcolsep}{10pt}
\centering
\begin{tabular}{l | l  c  c  c  c}
\toprule
\rowcolor{gray!20}
\bf Model	& \bf Method	& \bf IN-1K	& \bf IN-A	& \bf IN-R	\\
\midrule
\multirow{3}{*}{InternVL-C} & Dense			& 83.2		& 83.8		& 95.5		\\
& ShortGPT		& 79.7		& 57.9		& 90.4		\\
& FlattenGPT	& 81.6		& 74.6		& 93.7		\\
\bottomrule
\end{tabular}
\end{table}

\subsection{Generalization beyond Transformer Architecture}
Considering the various architectures available, it is far beyond the scope of this paper.
Yet we can provide an analysis of the generalization of our method. Since most architectures use skip connections, the flattening stage is very general and should work on these architectures as well.
However, there are not always appropriate channel pruning methods for these architectures.
If there is an appropriate channel pruning method, our method would work on various architectures.
Besides, transformer is a widespread baseline for many tasks, and our experiments on multiple transformer architectures and tasks have shown the effectiveness of our method.

\section{Limitation}
\label{sec:limitation}

FlattenGPT provides a novel approach for fine-grained LLMs depth compression, yet there are still some limitations.
First, FlattenGPT is performed on uniform architectures, where flattening will not change the model architecture significantly.
It is not trivial to compress the hybrid architectures, such as a combination of transformer~\cite{attn_is_all_you_need} and mamba~\citep{mamba}.
However, it is still worth researching the fine-grained depth compression method, as layer pruning methods operate on a very high granularity and cause performance degradation.
Second, we use one of the channel pruning methods to implement our FlattenGPT, while our framework is not constrained to specific channel pruning methods.
Developing better channel pruning methods will improve our depth compression method as well.


\end{document}